%% file: main.tex
\documentclass{article} % For LaTeX2e
\usepackage{iclr2025_conference,times}

\input{math_commands.tex}

\usepackage{hyperref}
\usepackage{url}
\usepackage{amsmath} % For mathematical symbols and equations
\usepackage{amsfonts} % For math fonts
\usepackage{amssymb} % For additional math symbols
\usepackage{graphicx} % For including images
\usepackage{wrapfig}
\usepackage[ruled,vlined]{algorithm2e}
\usepackage{hyperref} % For hyperlinks
\usepackage{cite} % For citations
\usepackage{booktabs} % For better table formatting
\usepackage{fancyhdr} % For custom headers and footers
\usepackage{color}
\usepackage{amsthm}
\usepackage{comment}
\usepackage{inconsolata}
\usepackage{subfig}
\usepackage{multirow}

\theoremstyle{plain}
\newtheorem*{theorem*}{Theorem}
\newtheorem{theorem}{Theorem}[section]

\newtheorem{lemma}[theorem]{Lemma}

\theoremstyle{definition}

\newtheorem{assumption}[theorem]{Assumption}

\newcommand{\methodname}{\textsc{ZETA}}

\title{\methodname: Leveraging $Z$-order Curves for \\ Efficient Top-$k$ Attention}
\author{Qiuhao Zeng\textsuperscript{$\spadesuit$}\quad 
    Jerry Huang\textsuperscript{$\heartsuit\diamondsuit$}\quad 
    Peng Lu\textsuperscript{$\heartsuit$}\quad 
    Gezheng Xu\textsuperscript{$\spadesuit$} \\
    \textbf{Boxing Chen}\textsuperscript{$\clubsuit$}\quad 
    \textbf{Charles Ling}\textsuperscript{$\spadesuit	\bigstar$}\quad 
    \textbf{Boyu Wang}\textsuperscript{$\spadesuit
    \bigstar$}\thanks{Corresponding author: Boyu Wang.} \\
    \textsuperscript{$\spadesuit$}University of Western Ontario\quad\textsuperscript{$\heartsuit$}Universit\'{e} de Montr\'{e}al\quad\textsuperscript{$\diamondsuit$}Mila\\\textsuperscript{$\clubsuit$}Noah's Ark Lab	\quad \textsuperscript{$\bigstar$}Vector Institute
}

\iclrfinalcopy % Uncomment for camera-ready version, but NOT for submission.
\begin{document}

\maketitle

\begin{abstract}
Over recent years, the Transformer has become a fundamental building block for sequence modeling architectures. Yet at its core is the use of self-attention, whose memory and computational cost grow quadratically with the sequence length $N$, rendering it prohibitively expensive for long sequences. A promising approach is top-$k$ attention, which selects only the $k$ most relevant tokens and achieves performance comparable to vanilla self-attention while significantly reducing space and computational demands. However, causal masks require the current query token to only attend to past tokens, preventing existing top-$k$ attention method from efficiently searching for the most relevant tokens in parallel, thereby limiting training efficiency. In this work, we propose \methodname, leveraging \textbf{Z}-Order Curves for \textbf{E}fficient \textbf{T}op-$k$ \textbf{A}ttention, to enable parallel querying of past tokens for entire sequences.
% in both space and time complexity of $\mathcal{O}(N \log N)$. 
We first theoretically show that the choice of key and query dimensions involves a trade-off between the curse of dimensionality and the preservation of relative distances after projection. In light of this insight, we propose reducing the dimensionality of keys and queries in contrast to values and further leverage $Z$-order curves to map low-dimensional keys and queries into \emph{one}-dimensional space, which permits parallel sorting, thereby largely improving the efficiency for top-$k$ token selection. Experimental results demonstrate that \methodname~matches the performance of standard attention on the synthetic \textsc{Multi-Query Associative Recall} task and outperforms attention and its variants on \textsc{Long Range Arena} and \textsc{WikiText-103} language modeling.

\end{abstract}

\section{Introduction}
Transformers~\citep{vaswani2017attention} have become indispensable for sequence modeling across various domains~\citep{openai_gpt-4_2023,zeng2024towards,zeng2024latent,zeng2023foresee,fang2022structure,fang2025benefits}, including natural language processing (NLP)~\citep{devlin2019bert,brown2020language, openai_gpt-4_2023, jiang_mixtral_2024}, computer vision~\citep{dosovitskiy2020image, DALL-E, sora}, etc. The foundation of Transformer models is the self-attention mechanism. This mechanism~\citep{bahdanau2016nmt}, inspired by recurrent neural networks (RNNs) and their ability to construct representations from all elements in a sequence, has revolutionized numerous fields, enabling breakthroughs in tasks such as language modeling~\citep{radford2019language}, machine translation~\citep{ott2018scaling}, text generation~\citep{brown2020language}, image classification~\citep{touvron2021training} and video generation~\citep{sora}. However, self-attention has a quadratic complexity in both memory and computation as the sequence length $N$ increases, which presents a significant challenge when scaling to long sequences~\citep{child2019generating, beltagy2020longformer}. This makes the direct application of self-attention in large-scale problems computationally prohibitive for many real-world applications, particularly when long sequences are involved~\citep{tay2020long}.

Recent advances have explored strategies to mitigate the inefficiencies of vanilla self-attention. One such approach is top-$k$ attention, which focuses computation on a subset of the most relevant tokens, significantly reducing memory and computation costs while maintaining competitive performance~\citep{kitaev2020reformer, gupta2021memory, bertsch2023unlimiformer, mao2024iceformer}. However, existing top-$k$ attention methods
% face several key challenges: i) Reformer
\citep{kitaev2020reformer,roy2021efficient} typically apply causal masking after selecting the top-$k$ tokens, causing earlier tokens ($i \ll N$) to often attend to nothing as their top-$k$ relevant tokens may include future tokens that are masked out by the causal mask.
% ii) Top-$k$ Memory Transformers~\citep{gupta2021memory} still require computation of the full attention matrix, resulting in an $\mathcal{O}(N^2)$ complexity; iii) 
% These methods process long sequences inefficiently with causal masks, as they require sequential processing of tokens. More specifically, models like 
Alternatively,~\citet{mao2024iceformer} process input tokens by token to exclude masked-out keys from the $k$-nearest neighbors search, preventing future tokens from influencing currently generated ones. Consequently, with causal masks, current top-$k$ attention approaches fail to fully leverage the parallel computation abilities of modern accelerators, limiting their efficiency in long-sequence modeling.

To overcome the limitations of existing top-$k$ attention methods, we introduce \methodname, a novel model designed to search for the top-$k$ tokens within chunked one-dimensional (one-dimensional) sorted key sequences projected via $Z$-order curves.
% , our approach enables efficient parallel processing across entire sequences. 
Specifically, our approach strikes a balance between mitigating the "curse of dimensionality" and preserving the relative distance of token represntations after projection by carefully selecting a lower dimensionality for keys and queries in contrast to values. This reduction allows the query and key to map to a one-dimensional space using $Z$-order curves, preserving proximity. As a result, \methodname~efficiently performs top-$k$ token selection in parallel within this one-dimensional space with accelerators.
% TODO:draw a figure, large dot production gives errors
% Additionally, we use a simple illustrative example to demonstrate that the Euclidean metric is more suitable for top-$k$ attention in low-dimensional spaces than the traditional dot-product metric. However, directly replacing dot-products with negative Euclidean distances degrades performance due to the negative signs. To address this, 
Additionally, since the aforementioned top-$k$ search is based on the Euclidean metric for low dimensional data, directly applying the traditional dot-product based softmax function is not appropriate. To address this, we propose an Adaptive Cauchy-Softmax mechanism that replaces the exponential function in the attention operation with a trainable Cauchy kernel~\citep{Bill86}. This enables dynamic adjustment of receptive fields across layers, providing greater flexibility in capturing both short and long-range dependencies.

Extensive empirical evaluations show that \methodname~matches the performance of standard self-attention on the Associative Recall task and consistently outperforms existing attention variants on the Long-Range Arena (LRA) and WikiText-103 datasets. We summarize our key contributions as follows:
%\methodname introduces the following key contributions:
% addresses the computational bottlenecks and limitations of sequential training in top-$k$ methods, offering an efficient and scalable solution for long-sequence tasks. In all, we illustrate the contributions of our proposed 1DFormer as follows:
\begin{itemize}
    \item \textbf{Efficient Parallel Top-$k$ Attention}: We introduce \methodname, a novel model that enables top-$k$ attention to operate in parallel across entire sequences, significantly improving training and inference efficiency with a time complexity of $\mathcal{O}(N \log N)$.

    \item \textbf{Dimensionality Selection for Key and Query Pairs}: We theoretically show that the dimensionality of keys and queries decides the trade-off between the curse of dimensionality and the preservation of relative distances for keys and queries.
    % , which can be used to improve computational efficiency.

    \item \textbf{$Z$-order Curve Integration}: By leveraging $Z$-order curves, we enable efficient top-$k$ token selection in one-dimensional space, allowing the use of parallel sorting algorithms on GPUs for faster attention computation.

    \item \textbf{Adaptive Cauchy-Softmax Mechanism}: We introduce Adaptive Cauchy-Softmax, a Softmax variant with trainable parameters based on the Cauchy kernel, dynamically adjusting receptive fields to enhance attention's flexibility.
\end{itemize}
% These contributions collectively enable scalable and efficient sequence modeling without sacrificing the performance of standard self-attention mechanisms.

\section{Related Works}
\textbf{Efficient Transformer} The Transformer architecture~\citep{vaswani2017attention} is foundational for sequence modeling, but its quadratic complexity limits efficiency with long sequences. Various efficient variants~\citep{survey,tay2020sparse,Skyformer,qincosformer,zhang2024the} have been proposed as alternatives, mainly categorized into sparse, low-rank, and linear transformers. Sparse transformers, such as BigBird~\citep{zaheer2020big} and Longformer~\citep{beltagy2020longformer}, restrict attention to local windows or global tokens to achieve linear complexity. SparseAxial~\citep{ho2020axial} further enhances this by combining sparse attention with axial mechanisms for high-dimensional inputs. Reformer~\citep{kitaev2020reformer}  locality-sensitive hashing to handle variable-length sequences efficiently. Low-rank transformers like Linformer~\citep{wang2020linformer} reduce the attention matrix to a lower-dimensional space, reducing memory and computation costs. Linear transformers such as Performer~\citep{choromanski2020rethinking} use kernel-based approximations for linear-time complexity, while Nyströmformer~\citep{xiong2021nystromformer} leverages Nyström decomposition for near-linear performance.

% {\color{red} Deja Vu: sparse attention. Adaptive Attention for Sparse-based Long-sequence Transformer; space complexity is O(N)!!}
\textbf{Top-$k$ Attention}~\citep{gupta2021memory} falls under the category of sparse attention, reducing attention complexity by selecting only the top-$k$ most relevant tokens at each layer, thereby focusing computational resources on the most critical interactions. Unlimiformer~\citep{bertsch2023unlimiformer} enables transformers to handle arbitrarily long sequences by chunking inputs and using a retrieval mechanism to attend to relevant past contexts. Similarly, IceFormer~\citep{mao2024iceformer} improves transformer efficiency by integrating a $k$-nearest-neighbor (\textsc{kNN}) search mechanism that focuses on the \textsc{kNN} results as the most relevant tokens during inference, bypassing the need to compute the full attention matrix. However, with causal masks, these approaches can not compute the outputs of a long sequences in parallel, making them less efficient for training models from scratch by not fully exploiting the parallel computing power of GPUs. In contrast, \methodname~performs \textsc{kNN}-based searches for relevant tokens in parallel across the entire sequence on GPUs using chunking techniques, enabling efficient training and inference with a time and space complexity of $\mathcal{O}(N\log N)$.

\section{Methodology}\label{sec:methodology}

\subsection{Preliminaries}\label{sec:preliminaries}
\textbf{Attention} has proven to be a fundamental building block in modern deep learning, particularly in natural language processing and sequence modeling tasks. It allows a model to focus on specific parts of the input sequence, thereby capturing dependencies within the sequence more effectively. In the standard formulation of attention~\citep{vaswani2017attention}, given a set of queries $\mQ\in \mathbb{R}^{N \times d_Q}$, keys $\mK\in \mathbb{R}^{N \times d_K}$, and values $\mV\in \mathbb{R}^{N \times d_V}$, the attention scores are computed as
\begin{equation}
    \text{Attention}(\mQ, \mK, \mV) = \text{softmax}\left(\mQ\mK^T/\sqrt{d_K}\right) \mV,
\end{equation}
% \begin{align}
% \label{eq: SA}
%     \text{SA}(\vq_i, \{\vk_j\}, \{\vv_j\}) = \sum_j\frac{\text{exp}(\vq_i^\top\vk_j/\sqrt{d})}{\sum_{j'}\text{exp}(\vq_i^\top \vk_{j'}/\sqrt{d})}\vv_j.
% \end{align}
% where $\vq_i, \vk_i, \vv_i$ are $d$ dimensional vectors.
where $d_K$, $d_Q$ and $d_V$ is the dimension of the keys, queries and values, respectively. It is common to use the same value for all three.

\noindent\textbf{Top-$k$ Attention} approaches aim to enhance the efficiency and flexibility of traditional self-attention mechanisms by focusing attention to only the most relevant tokens in a sequence. Both methods reduce the computational overhead associated with self-attention, especially for long sequences, by selectively attending to the most important tokens rather than computing attention scores across the entire sequence. This process lowers the computational complexity from $\mathcal{O}(N^2)$ in traditional self-attention to $\mathcal{O}(N \cdot k)$, where $N$ is the sequence length and $k$ is the number of selected tokens.
% In the context of top-$k$ attention, the attention mechanism operates by selecting only the most relevant tokens, reducing computational complexity while maintaining the effectiveness of attention.
The top-$k$ highest scores are selected for each query:
\begin{equation}
I_q = \left\{ i \mid \vq \mK_i^T / \sqrt{d_K} \geq \tau_k \right\}
\end{equation}
where $I_q$ represents the set of indices corresponding to the top-$k$ highest attention scores for the query $q$, $\tau_k$ is the threshold defined as the $k$-th highest attention score. The attention is then restricted to the tokens in this subset, reducing the number of operations required for long sequences. Specifically, the attention score for the top-$k$ tokens is recalculated using the self-attention mechanism but limited to the selected indices:
\begin{equation}
\text{Attention}_{\texttt{top-}k}(\vq, \mK, \mV) = \sum_{i\in I_q}\text{softmax} \left( \vq \mK_{i}^T / \sqrt{d_K} \right) \mV_i
\end{equation}
where $\vq$ denotes the query vector, $\mK_{i}$ and $\mV_i$ denote the key and value vectors corresponding to the top-$k$ relevant tokens, respectively.

% \noindent\textbf{$Z$-order Curve}~\citep{dugundji1989topology} also known as Morton order, is a method of mapping multidimensional data to one dimension while preserving the locality of the data points. This property makes it suitable for efficiently finding nearest neighbors in a high-dimensional space.
% The $Z$-order curve interleaves the bits of the coordinates of points in multidimensional space to produce a single scalar value that preserves the locality of the points. Given a point $\vx = (x_1, x_2, \ldots, x_d) \in \mathbb{R}^d$, each coordinate $\vx_i$ is first normalized and then discretized into an integer grid. The $Z$-order value is then computed by interleaving the bits of these integer coordinates. 

\noindent\textbf{$Z$-order Curves}~\citep{dugundji1989topology}, also known as Morton codes, provide a way to map multi-dimensional data into a one-dimensional space while preserving locality, whereas other dimensionality reduction methods~\citep{abdi2010principal,mcinnes2018umap} are not designed for mapping data to a 1D space. This approach is valuable in tasks that require efficient spatial indexing or key-query matching, such as attention. By maintaining the relative proximity of data points after projection, $Z$-order curves ensure that points that are close together in the original multi-dimensional space remain close in the projected one-dimensional space.

The $Z$-order curve interleaves the binary representations of each coordinate in a multi-dimensional point. For a point in the $d$-dimensional space with coordinates $\vx=(x_1, x_2, \dots, x_d)$, where each $\vx_i$ is a binary number, the $Z$-order curve computes a scalar value $Z$ by interleaving the bits of each coordinate. Given the binary representation of $x_i$ as $b_{i1}b_{i2} \dots b_{in}$, the $Z$-order curve is expressed as: 
\begin{equation}
Z = b_{11}b_{21} \dots b_{d1}b_{12}b_{22} \dots b_{d2} \dots b_{1n}b_{2n} \dots b_{dn}
\end{equation}
where $n$ refers to the number of bits used to represent each coordinate $x_i$ in its binary form. Through this interleaving of bits, the $Z$-order curve creates a scalar value that allows efficient sorting or indexing of points while approximately maintaining their original spatial relationships.

The primary advantage of $Z$-order curves is their ability to preserve locality. In other words, nearby points in the original multi-dimensional space have similar $Z$-values in the projected one-dimensional space. This property enables efficient search and selection processes in attention mechanisms or spatial indexing, where key-query pairs can be processed more efficiently in the one-dimensional space without significantly losing the locality information from the higher-dimensional space. {\color{black}$Z$-order curves are designed to preserve locality, and hence not suitable for dot-product similarity measures, which does not reflect locality.}

% Consider a 2D example with coordinates $(x, y)$. The process is as follows:
% \begin{itemize}
%     \item[1] Normalization and Discretization 
% \begin{equation}
%    \hat{x} = \left\lfloor \frac{(x - \min(x)) \cdot (2^b - 1)}{\max(x) - \min(x)} \right\rfloor, \quad \hat{y} = \left\lfloor \frac{(y - \min(y)) \cdot (2^b - 1)}{\max(y) - \min(y)} \right\rfloor,
%    \end{equation}
%    where $b$ is the number of bits for discretization, and $\min(x)$, $\max(x)$, $\min(y)$, and $\max(y)$ are the minimum and maximum values of $x$ and $y$ in the dataset.
%    \item[2] Interleaving Bits: The bits of $\hat{x}$ and $\hat{y}$ are interleaved to form a single $Z$-order value. For example, if $\hat{x} = 01$ and $\hat{y} = 10$, the interleaved result is $0101$.

%    \item[3] Generalization to Higher Dimensions:  For a point $\mathbf{x} \in \mathbb{R}^d$, the bits of all $d$ coordinates are interleaved to form the $Z$-order value.
% \end{itemize}

\subsection{Searching for the Top-$k$ Attended Tokens in One-Dimensional Space}
Since we project key and query vectors into a one-dimensional space using $Z$-order curves, using a large $d_K$ can still distort locality (as shown in \autoref{fig:locality_preservation}) and compromise the preservation of relative distances. Thus we ask whether $d_K$ can be reduced in a way such that even after mapping to one dimension, relative distances between tokens are maintained. Importantly, the key and query dimensions $d_K$ and $d_Q$ do not have to match the dimension of the values $d_V$. This is because $d_V$ should remain large to capture more semantic information, as seen with Gaussian distributions, where higher dimensionality increases the measure of information entropy~\citep{thomas2006elements}. Hence, as long as the relative distances between queries and keys are preserved, $d_K$ and $d_Q$ can be reduced.
% , significantly reducing computational costs. 
% Since the time complexity of a single attention operation is $\mathcal{O}(N^2 d_K^2)$, by setting $d_K$ to a smaller value, such as $d_K = \frac{d_V}{8}$, the time complexity can be reduced to $\frac{1}{64}$ of the original, yielding substantial computational savings. 
The following theoretical analysis provides insights into the selection of $d_K$. 
% considering the curse of dimensionality.

\subsubsection{Theoretical Analysis on $d_K$}
We first introduce the Johnson–Lindenstrauss Lemma~\citep{johnson1986extensions}, which states that data in high-dimensional space can be projected into a much lower-dimensional subspace using random projections while approximately preserving the pairwise distances between the points. Since random projections can preserve locality, this provides justification for setting a smaller $d_K$ with trainable projection functions for keys and queries, which could also preserve locality.

\begin{lemma}
\label{lemma:jl}
\textbf{(Johnson–Lindenstrauss Lemma)} For any $0 < \epsilon < 1$ and any integer $m$, let $d$ be a positive integer such that $d = \Omega (\frac{ \ln m}{\epsilon^2})$. Then for any set $x$ of $m$ points in $\mathbb{R}^D$, there exists a map $f: \mathbb{R}^D \to \mathbb{R}^d$ such that for all $x_i, x_j \in \mathcal{X}$,
\begin{align}
(1 - \epsilon) \|x_i - x_j\|^2 \leq \|f(x_i) - f(x_j)\|^2 \leq (1 + \epsilon) \|x_i - x_j\|^2
\end{align}
\end{lemma}

The following assumption then provides a mathematical depiction that attention weights are constrained within an $m$-dimensional simplex, and the learnable similarity function $\Gamma$ outputs the attention scores, ensuring the most relevant tokens are emphasized during the information aggregation process. This reflects the primary goal of attention: to aggregate critical information for more accurate predictions. 
% By learning to weight key tokens effectively, the attention mechanism improves its capacity to focus on the most informative parts of the input sequence.

% This following assumption establishes a mathematical framework for the attention mechanism, where the attention weights are constrained to lie in an $m$-dimensional simplex. 
% It further assumes that the model can learn an optimal similarity critic function to compute attention scores, which are then used to minimize the expected risk, ultimately improving the accuracy of the attention-based predictions.
\begin{assumption}\label{assum:optimal-attn}
    Let $\alpha \in \Delta_{m-1}$ be an element of the $m$-dimensional simplex, defined as $\Delta_{m-1} \triangleq \left\{\alpha \in \mathbb{R}^m \mid \alpha_i \geq 0, \sum_{i=1}^{m} \alpha_i = 1\right\}$. Assume that $h_{\mathrm{attn}}$ equipped with $\alpha$ can achieve an optimal learnable similarity critic function ${\Gamma}$, where the attention scores are given by $\alpha = \mathrm{softmax}\left(\Gamma\left(f_k(x_i), f_q(x)\right)\right)$, such that $\Gamma$ is trained to be optimal to have the minimal expected risk: $ \min_{\alpha}\|h_{\mathrm{attn}}(x,S_x;\alpha) - y\|$, where $h_{\mathrm{attn}}$ denotes the attention-based hypothesis, $x$ is the input, and $S_x$ is the context.
\end{assumption}
% Above assumption assumes that the learned similarity function $\gamma$ can make the attention scores $\alpha$ minimize the expected risk.

% The following lemma provides insight into the probability that subsets are disjoint based on the size of the sample set, and it plays a crucial role in establishing the result of our final theorem.
\begin{comment}
\begin{lemma}\label{lemma:monotone}
    Let $\alpha\in \mathbb{R}^m$, $\forall h_1:\mathbb{R}^m\rightarrow \mathbb{R}$ and $\forall h_2:\mathbb{R}^m\rightarrow \mathbb{R}$. Assume that $ h_1(\alpha)\le h_2(\alpha)$, and then $\min_{\alpha} h_1(\alpha)\le \min_{\alpha} h_2(\alpha)$.
\end{lemma}
\end{comment}
% \begin{lemma}\label{lemma:sets-intersect}Let $C_1, \ldots, C_r$ be a collection of subsets of some domain set, $\mathcal{X}$. Let $S$ be a sequence of $m$ points sampled i.i.d. according to some probability distribution, $\mathcal{D}$, over $\mathcal{X}$. Then,
% $$
% \underset{S \sim \mathcal{D}^m}{\mathbb{E}} \left[ \sum_{i: C_i \cap S = \emptyset} \mathbb{P}[C_i] \right] \leq \frac{r}{m e}.
% $$
% \end{lemma}
\autoref{thm:main} highlights the importance of choosing $d_K$ carefully, as it controls for a trade-off between locality preservation and the curse of dimensionality. Larger $d_K$ allow for more detailed feature capture at the cost of the high-dimensional curse, leading to increased complexity. On the other hand, a smaller $d_K$ loses locality between tokens, which is crucial for efficient query. The bounds provided give valuable insights into the underlying mechanisms of attention and can guide future designs of more efficient attention models. For simplicity, we assume WLOG that keys and queries share the same projection functions, as~\citet{kitaev2020reformer}.
% The following bounds 
% not only derive generalization guarantees for the current success of autoregressive transformer-based models, but also 
% provide insights into the underlying mechanisms of attention, which can guide future designs of more efficient attention mechanisms. For simplicity, we assume without loss of generality that keys and queries share the same projection functions (following~\citet{kitaev2020reformer}).
% , Reformer, proposed the shared-QK Transformer that shares the linear projection layer for keys and queries. It reduces the number of parameters and memory space while not affecting the model performance. https://arxiv.org/abs/2001.0445a very simple but efficient technique.
\begin{theorem}\label{thm:main}
Let $\mathcal{X} \in \mathbb{R}^d$, $\mathcal{Y} \in \mathbb{R}^D$, and $\mathcal{D}$ be a distribution over $\mathcal{X} \times \mathcal{Y}$ for which the conditional probability function, $h: \mathbb{R}^d \rightarrow \mathbb{R}^{D}$, is a $l$-Lipschitz function. Let $h$ denote a hypothesis, and $h_{\mathrm{attn}}$ denote the one-layer attention model to aggregate the predictions of a sample set $S \sim \mathcal{D}^m$ to predict for another i.i.d sample $x$. Specifically, here we assume the same linear map for a key mapping $f_k$ and a query mapping $f_q$ as $f:\mathbb{R}^d \rightarrow [-B,B]^{d_K}$ where $B$ is the bound of projection of $x$ by $f$, and we assume the value mapping $f_v$ be the Bayes optimal rule $h^*=\mathbb{E}[Y|X=x]$, which is the hypothesis that minimizes $L_{\mathcal{D}}(h)$ over all functions. Then, the expected risk of the regression tasks $L_{\mathcal{D}}(h) = \mathbb{E}_{(x,y) \sim \mathcal{D}}\|h(x) - y\|$ with $h$ as $h_{\mathrm{attn}}$ can be upper bounded by
$$\underset{S \sim \mathcal{D}^m}{\mathbb{E}} \left[L_{\mathcal{D}}(h_{\mathrm{attn}})\right] \leq L_{\mathcal{D}}(h^*) +\frac{4lc \sqrt{d_K}B  m^{-1/(d_K+1)}}{\sqrt{1-\sqrt{\frac{C\ln m}{d_K}}}}$$
\end{theorem}
This indicates that $d_K$ should be carefully chosen rather than simply being set as equal to $d_V$ (as is common practice). Empirically, we show that the dimension of query and key can decreased without degrading performance, more specifically in~\autoref{fig:dk_associative_recall} in~\autoref{sec:dimdk} which we later discuss.
% , and additionally, the top-$k$ samples remain consistent even after applying the $Z$-order curve transformation to one-dimensionalimension.

\subsubsection{Top-$k$ Search in One Dimension}

As the sequence length becomes extremely large, iterating through the entire context history to search for the top $k$ tokens becomes infeasible. Ideally, searching should be efficient and aim to achieve the optimal time complexity of $\mathcal{O}(N \log N)$, similar to general sorting problems with arbitrary inputs. To achieve this, we map both keys and queries into a one-dimensional space using the $Z$-order curve and sort them with a sorting operation that can be executed in parallel on an accelerator, e.g. the \texttt{torch.sort} operator in PyTorch~\citep{pytorch}. The insertion position of a query in the key sequence can then be found using a binary search (e.g. with \texttt{torch.searchsorted}), allowing us to retrieve the top-$k$ attended tokens using a window centered on the insertion position.

Specifically, the key and query dimensions are set to be significantly smaller than that of the values, i.e. $d_K = d_Q \ll d_V$. While values need a high dimensionality to carry rich semantic information, keys and queries primarily serve to preserve relative distances, which can be achieved with much lower dimensionality as argued by Lemma~\ref{lemma:jl}. To facilitate the fast retrieval of queries from keys, we leverage sorting, which can be efficiently parallelized, after mapping queries and keys into one-dimensional space via the $Z$-order curve. We define $\mQ, \mK \in \mathbb{R}^{B \times N \times d_K}$, where $B$ is the batch size and $N$ is the sequence length. The $Z$-order transformation is applied as follows: 
\begin{equation*}
    \mQ_z = \text{$Z$-order}(\mQ), \quad \mK_z = \text{$Z$-order}(\mK)   
\end{equation*}
where $\mQ_z$ and $\mK_z$ are the one-dimensional representations of $\mQ$ and $\mK$, respectively.

With causal masks, directly collecting the top-$k$ tokens from the entire sequence is not guaranteed to have plenty of tokens to make inference. For instance, for a query at position 32 in a sequence of length 2048, selecting the top 16 tokens from the entire sequence followed by causal masking would leave approximately $\frac{32}{2048}\times 16=\frac{1}{4}$ tokens to attend to, effectively leaving no tokens available for the query. Consequently, the query at this position would not attend to anything, rendering top-$k$ attention ineffective. To enable parallel $k$-nearest neighbors ($k$NN) searching considering causal masks, we first sort the $Z$-order keys and divide them into chunks. For query $i$ in the $m$-th chunk (where $m = \lfloor i / M \rfloor$ and $M$ is the chunk size), we restrict search to the first $m$ chunks, indexing the original unsorted keys from 0 to $m \times M - 1$ in the sorted list. This ensures that future keys $j \in \{j : j > m \times M\}$ are excluded, in accordance with the causal mask requirements. This process is performed in parallel for every query.

Next, we perform the nearest neighbor search in these one-dimensional $Z$-order spaces. For each query, the insertion position is first found using a binary search and followed by selecting the nearest keys using a window of size $K$ centered around the insertion point to collect top-$k$ tokens as $I_q$, denoting the indices of the $k$-nearest neighbors. This ensures efficient and accurate retrieval while maintaining the constraints imposed by causal masking.

\subsection{Adaptive Cauchy-Softmax}

% Two common metrics are negative Euclidean distance~\citep{mccarter2023inverse} and dot product similarity. 
% Top-$k$ token retrieval requires a similarity measure between data points. We prefer the Euclidean metric for top-$k$ attention with a small $d_K$ for two reasons. 
% First, we posit the Euclidean metric is more effective for top-$k$ attention for low-dimensional data through an example (\autoref{fig:wrapped_figure}). This example demonstrates that Euclidean distance is more reliable for classification tasks in one-dimensional space for the top-$k$ methods, with the negative Euclidean distance-based metric correctly identifying the class, leading to a more accurate prediction, while the dot product can be misleading, leading to incorrect predictions. The second is that the $k$-NN search is normally applied with Euclidean metric, dot production needs the undesirable normalization operation, which loses the magnitudes of the tokens.
% top-$k$ attention 
Top-$k$ token searching relies on a similarity metric between data points, and we prefer the Euclidean metric for top-$k$ attention with small $d_K$ for two main reasons. First, as illustrated in Figure~\ref{fig:wrapped_figure}, Euclidean distance is more effective for low-dimensional data in top-$k$ methods: it reliably identifies the correct class in one-dimensional classification tasks, leading to accurate predictions, whereas the dot product can be misleading (samples in the 
\begin{wrapfigure}{r}{0.45\textwidth} % 'r' for right, 'l' for left
    \centering
    % \vspace{-1cm}
    \includegraphics[width=\linewidth]{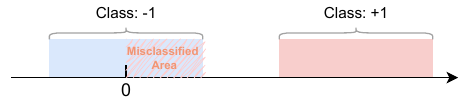} % Replace with your image path
    \caption{Illustration of attention using Euclidean distance vs. dot product. Euclidean distance correctly classifies points into classes $\pm$1, while the dot product leads to a misclassified area.}
    \label{fig:wrapped_figure}
    % \vspace{-0.5cm}
\end{wrapfigure} \ misclassified area will be classified as ``+1" using the dot-product metric). Second, $k$-NN search is typically based on the Euclidean metric, while using the dot-product requires normalization that loses token magnitudes.

To better align with the Euclidean measure for low-dimensional representations, we propose the Adaptive Cauchy Softmax function, to replace the exponential function in traditional attention mechanisms. The Cauchy kernel, with its heavier tails, ensures that distant tokens retain influence, overcoming the limitations of the exponential function, which suppresses distant tokens~\citep{relu_attention}. This allows the attention mechanism to capture both local and global dependencies, with the shape of the kernel determining how key vectors influence the query.
% In traditional attention mechanisms, the softmax function employs an exponential function, which suits the dot-product between query and key vectors. However, to introduce more flexibility and adaptiveness in the attention mechanism, we propose replacing the exponential function of dot-productions with a Cauchy kernel function, which aligns more effectively with the Euclidean metric. 
Specifically, the Adaptive Cauchy-Softmax between a query vector $\vq$ and keys $\mK$ is computed as:
\begin{equation}\label{eq:cauchy}
\text{softmax}_c(\vq, \mK) =\frac{ \frac{\gamma}{\pi}  \left[\|\vq - \mK_i\|^2 + \gamma^2\right]^{-1}}{\sum_{j\in I_q}\frac{\gamma}{\pi}  \left[\|\vq - \mK_j\|^2 + \gamma^2\right]^{-1}}=\frac{  \left[\|\vq - \mK_i\|^2 + \gamma^2\right]^{-1}}{\sum_{j\in I_q} \left[\|\vq - \mK_j\|^2 + \gamma^2\right]^{-1}}
\end{equation}
where $\gamma$ is a trainable parameter that controls the shape of the distribution. By training a task-specific $\gamma$ for each attention layer, the model adjusts receptive fields dynamically. We define $\gamma^2$ as the output of a sigmoid function applied to a trainable parameter, to ensure a range of $[0,1]$, with smaller values sharpening attention and improving focus on relevant inputs~\citep{devil_linear_attn, zhang2024the}, while larger values allow for smoother attention. The adaptive Cauchy softmax effectively handles long-range dependencies, preventing entropy collapse ~\citep{zhai2023stabilizing} or explosion~\citep{zhang2024the} and adaptively balancing attention across the sequence.

\subsection{Sparse Attention with $Z$-order Curve for Efficient \textsc{kNN} Retrieval}

Instead of calculating full attention scores, which is computationally expensive and memory intensive, we compute sparse attention scores by leveraging the $Z$-order curve for efficient nearest neighbor retrieval. $\methodname$ is then computed as below, according to \autoref{eq:cauchy}:
\begin{equation}
\mathrm{Attention}_{\texttt{\methodname}}(\mQ, \mK, \mV) = \sum_{i \in I_q} \mathrm{softmax}_c\left(\mQ,\mK_i\right) \mV_i
\end{equation}
where $\mK_i$ and $\mV_i$ are the corresponding keys and values for the indices $i \in I_q$. As a result of the sparsity of top-$k$, most of the tokens will not join the predictions, which stops the gradient back-propagated through the low-probabilities tokens and fails to leverage this current prediction's information. We append the mean vector of the history tokens to the top-$k$ tokens matrix using cumsum function in $I_q$, to keep the attention from assigning zero probability, which can be regarded as smoothing in n-gram language model~\citep{jm3}.

\section{Experimental Results}
We evaluate \methodname’s performance on several aspects: \methodname’s ability to solve the synthetic \textsc{Multi-Query Associative Recall} (\textsc{MQAR}) task~\citep{arora2024zoology}, long sequence modeling ability on the \textsc{Long Range Arena} (\textsc{LRA}) benchmark and auto-regressive language modeling on \textsc{WikiText-103}. Then we conduct extensive analysis experiments: an ablation study examining the influence of dimensionality on attention model performance (\autoref{sec:dimdk}), ablations on various Euclidean-based Softmax operators (\autoref{sec:euclidean}), the empirical results of locality preservation using $Z$-order curves (\autoref{sec:locality}) and an ablation study over the number of $k$ in \methodname~(\autoref{sec:top-$k$}).

\subsection{Empirical Validation}
\paragraph{Associative Recall.} Associative recall tasks~\citep{arora2024zoology} have been popular for testing the ability of language models to look up information in
their context. Broadly, they involve feeding auto-regressive model pairs of key-value associations and then prompting the
model to produce the correct completion upon being shown a previously seen key. The \textsc{Multi-Query Associative Recall}
(\textsc{MQAR}) task is a particular formulation of this task that requires the model to memorize multiple associations~\citep{arora2024zoology}. We evaluate the performance of various models on the Associative Recall task, a classical sequence-to-sequence task that requires the model to recall the first token associated with a target token after processing a long sequence. The task is tests the ability of models to capture long-range dependencies and to maintain information over time.

We compare the performance of four models: a vanilla Transformer, Performer~\citep{choromanski2020rethinking}, \textsc{Based}~\citep{arorasimple}, and \methodname, with different model dimensions (32, 64, 128, and 256). As illustrated in \autoref{fig:associative_recall}, accuracy increases as the model dimension grows. Attention and Based models show strong performance with higher dimensions, achieving nearly perfect accuracy for dimensions larger than 64. \methodname follows a similar trend and achieves competitive performance, especially for larger model dimensions, with perfect accuracy at dimension 256. In contrast, the Performer struggles, showing significantly lower accuracy across all dimensions.

% This highlights the critical role of model capacity in capturing and maintaining long-range dependencies. While models like Attention, Based, and \methodname~benefit from increased model dimensions, Performer’s kernel-based attention appears less effective in handling the complexity of this task.

\begin{figure*}[t]
% \subfloat{\includegraphics[width=0.16666\textwidth]{fig/rainbow.pdf}}
\subfloat{\includegraphics[width=0.24\textwidth]{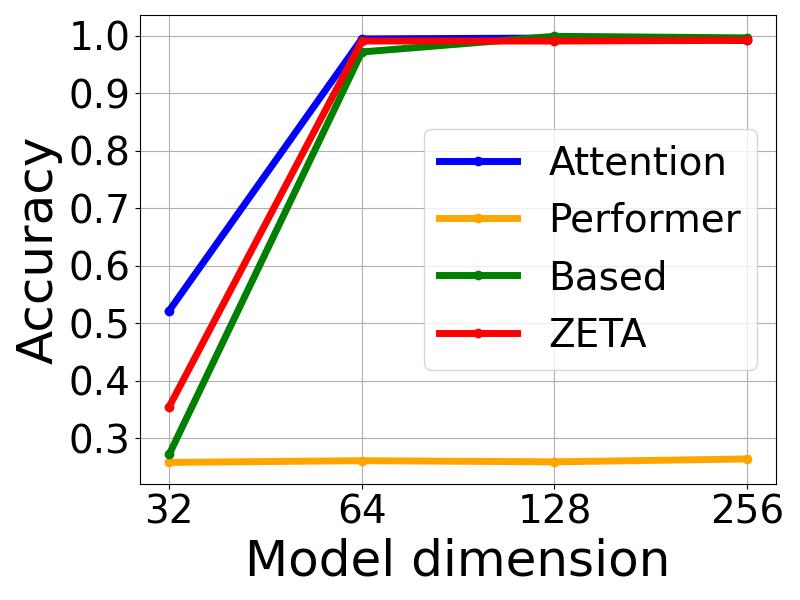}\label{fig:associative_recall}}
\subfloat{\includegraphics[width=0.24\textwidth]{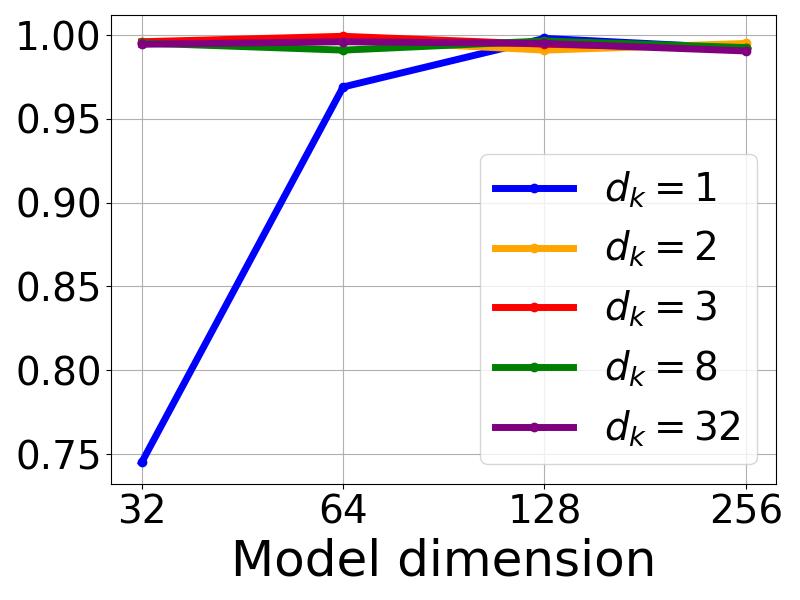}\label{fig:dk_associative_recall}}
\subfloat
{\includegraphics[width=0.24\textwidth]{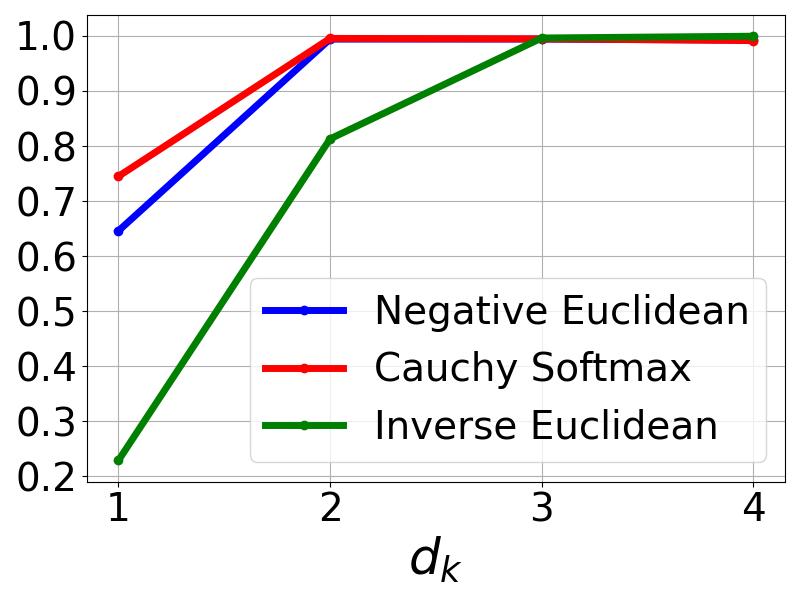}\label{fig:euclidean_softmax}}
\subfloat
{\includegraphics[width=0.24\textwidth]{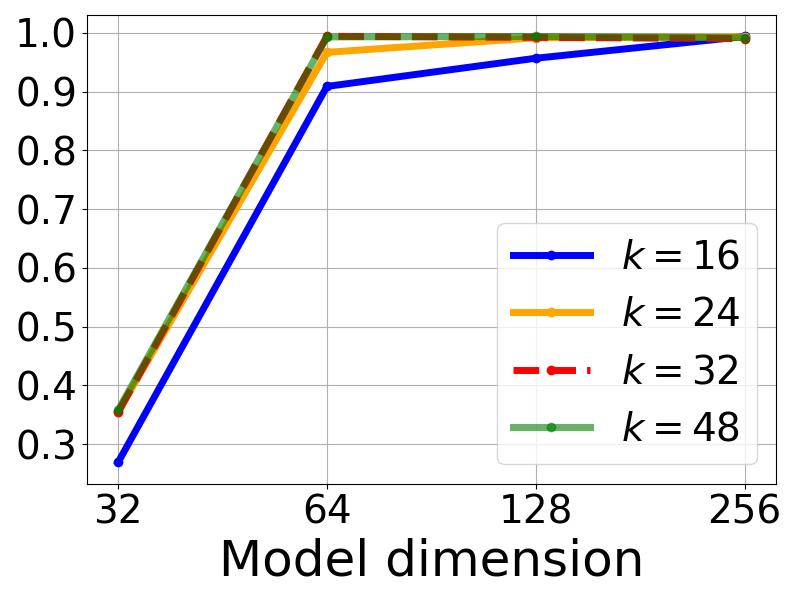}\label{fig:top_k_ablation}}
% (a) First subplot are different domains indexed with the rainbow color. 
\caption{Experiments on Associative Recall: (a) Model Accuracy (b) Performance of Transformer with varying $d_K$ across different model dimensions; even with low $d_K$, the model achieves near-perfect performance (c) Comparison of different Euclidean-based Softmax operators across varying key-query dimensions $d_K$ (d) Ablation on $k$ in \methodname.}
    \vspace{-0.3cm}
\label{fig:AR}
\end{figure*}

\paragraph{Long Range Arena (LRA).} The \textsc{Long Range Arena} (\textsc{LRA}) benchmark~\citep{tay2020long} is a comprehensive suite designed to evaluate the performance of models on long sequence tasks. It includes tasks that span across multiple domains, such as natural language processing, image classification, and mathematical reasoning. \textsc{LRA} focuses on sequence classification, challenging models to efficiently process longer input sequences while capturing long-range dependencies, providing an ideal testbed for Transformer models and their efficient variants.

\begin{wraptable}{r}{0.45\textwidth}
    \centering
    \caption{Test perplexity (lower is better) on \textsc{WikiText-103}.}
    \resizebox{\linewidth}{!}{
    \begin{tabular}{lcc}
        \toprule
        \textbf{Model} & \textbf{Params} & \textbf{Test PPL} \\
        \midrule
        Vanilla Transformer & 125M  & 26.2 \\
        Performer & 125M & 26.8 \\
        Reformer & 125M & 25.6 \\
        AFT-conv & 125M & 28.2 \\
        Linear Transformer & 125M & 30.2 \\
        RFA-Gaussian & 125M & 27.5 \\
        CosFormer & 125M &\textbf{23.1} \\
        \midrule
        \methodname~& 124M & 26.3\\
        \bottomrule
    \end{tabular}
    }
    \label{tab:perplexity_results}
\end{wraptable}
LRA consists of five key tasks: \textsc{ListOps}, \textsc{Text}, \textsc{Retrieval}, \textsc{Image}, and \textsc{PathFinder}. Each task evaluates different aspects of long-range dependency handling such as the ability to handle mathematical reasoning tasks on long sequences of operations (\textsc{ListOps}), capture dependencies over long textual inputs (\textsc{Text}), retrieve relevant elements from a long sequence (\textsc{Retrieval}) and capture spatial dependencies (\textsc{Image}). \textsc{PathFinder} presents a difficult problem where models must distinguish between connected and disconnected paths within maze-like patterns. A modified version of \textsc{PathFinder}, called \textsc{PathFinder-X}, is also included where the patterns are presented in a larger image (256$\times$256 compared to 32$\times$32) but has yet to be solved by existing attention-based methods.

We evaluate various Transformer-based models, including several linear and efficient variants, trained from scratch on the LRA sequence classification tasks. For each model, we adopt the same hyperparameter settings provided by the official LRA benchmark~\citep{tay2020long} to ensure a fair comparison. Results are summarized in \autoref{tab:lra_results}, which compares the performance of the models across all five tasks, along with their average accuracy, showing that \methodname~significantly outperforms other attention-based models.

\vspace{0.1cm}
\begin{table}[h]
    \centering
    \caption{Results of the Transformer and various variants on LRA. We consistently outperform the next closest competitor~\citep{zhu-soricut-2021-h}.}
    \resizebox{0.8\textwidth}{!}{
    \begin{tabular}{l|ccccc|c}
        \toprule
        \textbf{Model} & \textbf{ListOps} & \textbf{Text} & \textbf{Retrieval} & \textbf{Image} & \textbf{Pathfinder} & \textbf{Average} \\
        \midrule
        Transformer & 36.37 & 64.27 & 57.46 & 42.44 & 71.40 & 54.39 \\ 
        % Flash Attention & 37.64 & 63.93 & 81.42 & 43.59 & 72.73 & 59.84 \\
        Reformer & 37.27 & 56.10 & 53.40 & 38.07 & 68.50  & 50.67 \\
        Sparse Trans & 17.07 & 63.58 & 59.59 & 44.24 & 71.71 & 51.24 \\
        Sinkhorn Trans & 33.67 & 61.20 & 53.83 & 41.23 & 67.45 & 51.29 \\
        Linformer & 35.70 & 53.94 & 52.27 & 38.56 & 76.34 & 51.36 \\
        BigBird & 36.05 & 64.02 & 59.29 & 40.83 & 74.87 & 55.01 \\
        Linear Trans. & 16.13 & 65.90 & 53.09 & 43.40 & 75.30 & 50.76 \\
        Performer & 18.01 & 65.40 & 53.82 & 42.77 & \textbf{77.05}  &  51.41 \\
        % FNet & 35.33 & 65.11 & 59.61 & 38.67 & 77.80 & 54.42 \\
        Nyströmformer & 41.28 & 58.38 & 65.40 & 37.54 & 71.76 & 54.87 \\
        H-Transformer-1D & \textbf{49.53} & \textbf{78.69} & 63.99 & 46.05 & 68.78 & 61.41 \\
         % Nyströmformer & 37.15 & 65.52 & 79.56 & 41.58 & 70.94 & $\times$ & 57.46 \\
        % Luna-256 & 37.25 & 64.57 & 79.29 & 47.38 & 78.37 & $\times$ & 59.37 \\
        Top-$k$ Attention & 38.12 & 63.72 & 59.14 & $\times$ & $\times$ & 53.66 \\
        IceFormer & 41.53 & 60.01 & 66.02 & 40.46 & 74.42 & 56.49 \\ 
        cosFormer & 37.90 & 63.41 & 61.36 & 43.17 & 70.33 & 55.23 \\
        Skyformer & 39.25 & 64.70 & 82.06 & 40.77 & 70.73 & 59.50 \\
        Hedgehog & 37.15 & 64.60 &\textbf{82.24} & 40.15 & 74.16 & 59.66 \\
        % {\color{black} FlashAttention} & {\color{black}37.60} & {\color{black}63.90} & {\color{black}81.4} & {\color{black}43.5} & {\color{black}72.7} & {\color{black}59.8} \\
        \midrule
        \methodname & {42.52} & 64.52 & 77.92 & \textbf{64.39} & 68.20 &\textbf{63.51} \\
        % S4 & 59.60 & 86.82 & 90.90 & 88.65 & 94.20 & 96.35 & 86.09 \\
        \bottomrule
    \end{tabular}}
    \label{tab:lra_results}
    % \vspace{-0.5cm}
\end{table}

\paragraph{Autoregressive Language Modeling.} Furthermore, we evaluate several models on the \textsc{WikiText-103}~\citep{merity2016pointer}, a widely used benchmark for language modeling containing over 100 million tokens extracted from high-quality Wikipedia articles characterized by a large vocabulary and long-range dependencies. This makes it a challenging benchmark for testing a model's ability to predict the next token in a sequence. We use perplexity (PPL) as the primary evaluation metric, where lower scores indicate a better ability to capture the sequential structure within natural language text. \autoref{tab:perplexity_results} shows a vanilla Transformer\footnote{We use an auto-regressive Transformer based on \citet{biderman2023pythia} for comparision.} to achieve a test perplexity of 26.2. 
However, linear approximation models such as the Linear Transformer \citep{devil_linear_attn} struggle to compete, with higher perplexity values of 30.2 on the test set. 
% This highlights the limitations of linear models in effectively modeling long-range dependencies in large datasets like WikiText-103.
The table further compares several other models, including efficient attention mechanisms like Performer~\citep{choromanski2020rethinking}, Reformer~\citep{kitaev2020reformer}, and CosFormer~\citep{qincosformer}. Notably, CosFormer achieves the lowest perplexity on the test set with a score of 23.1, outperforming all other models. Reformer also shows competitive results, achieving a perplexity of 25.6, improving on the Vanilla Transformer. 
% In contrast, models like AFT-conv and Linear Transformer exhibit higher perplexity scores of 28.2 and 30.2, respectively, indicating weaker performance. 
\methodname~achieves a perplexity of 26.3, comparable to the Vanilla Transformer.

The results highlight the trade-offs between using conventional transformers, linear transformers, and models adopting approximate attention mechanisms like \methodname. It reinforces the importance of balancing computational efficiency with model performance, particularly in the context of long-sequence language modeling tasks, especially as the information necessary to solve a task becomes sparsely located within long contexts.

\subsection{Effect of Varying $d_K$ on Associative Recall Task}\label{sec:dimdk}

We next evaluate the effect of different key-query dimensions $d_K$ on the Transformer model’s performance for MQAR. The model dimensions are varied between $\{32, 64, 128, 256\}$ while adjusting $d_K$ to values $\{1, 2, 3, 8, 32\}$. As shown in \autoref{fig:dk_associative_recall}, the performance remains near-perfect even with low $d_K$ values, such as $d_K = 2$. The model achieves close to $100\%$ accuracy across all model dimensions, except for the smallest dimension ($d_K = 1$), where performance slightly drops for lower model dimensions. This demonstrates that the Transformer is capable of handling long-range dependencies in the Associative Recall task, even with relatively low key-query dimensionality.

These results suggest that reducing $d_K$ does not significantly impair the model's ability to recall information in sequence tasks. In fact, maintaining a low $d_K$ can provide computational savings without sacrificing performance, especially when model dimensions are large. This indicates that while random projections---such as those used in the Johnson-Lindenstrauss Lemma---approximately preserve distances, trainable projection networks $f_k$ and $f_q$ can better adapt to task-specific data and more effectively retain locality even with a low $d_K$. For instance, by setting $d_K$ as low as 3, we reduce it from the typical head dimension (normally 32, i.e. feature dimension as 512 with 8 heads). We can further mitigate information loss by configuring $f_k$ and $f_q$ as two-layer neural networks rather than single-layer ones.

\subsection{Performance of Euclidean-Based Softmax Operators}
\label{sec:euclidean}

We further evaluate the performance of transformers with various Euclidean-based Softmax operators on MQAR. Specifically, we compare Negative Euclidean, Cauchy Softmax (our proposed method), and Inverse Euclidean operators. The goal of this experiment is to test how these different formulations of Softmax handle varying key-query dimensions $d_K$ in terms of accuracy.

As shown in \autoref{fig:euclidean_softmax}, the proposed Cauchy Softmax consistently outperforms the other operators across all values of $d_K$. It achieves near-perfect accuracy for $d_K \geq 2$, whereas Negative Euclidean shows a drop in accuracy for lower $d_K$ values. Inverse Euclidean, while performing comparably at higher dimensions, struggles significantly at lower values of $d_K$ (e.g., $d_K = 1$).

These results highlight the advantage of using the Cauchy distribution for a smaller $d_K$, as it allows for better handling of long-range dependencies and achieves more stable performance across various key-query dimensions. The heavier tails of the Cauchy distribution enable distant tokens to retain non-negligible influence, which is crucial for tasks like Associative Recall where long-range token relationships are important.

\subsection{Locality Preservation after $Z$-order Curve Projection}
\label{sec:locality}
\begin{wrapfigure}{r}{0.4\textwidth} % 'r' for right, 'l' for left
    \centering
    % \vspace{-0.6cm}
    \resizebox{\linewidth}{!}{
        \includegraphics[width=\linewidth]{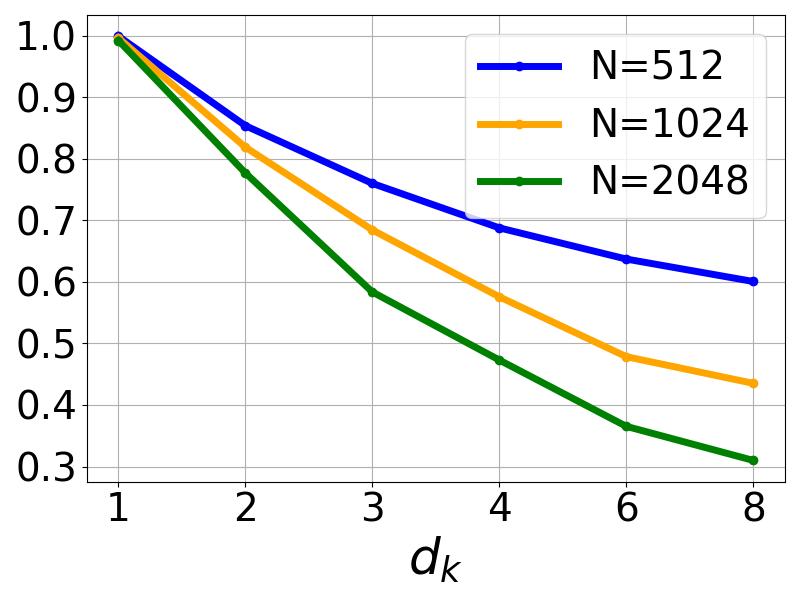} % Replace with your image path
    }
      \vspace{-0.5cm}
    \caption{The effect of dimensionality reduction before and after $Z$-order curves projection on locality preservation for different sample sizes.}
    % . The locality is measured by the overlap between the top-64 nearest neighbors before and after $Z$-order curve projection.}
    
    \label{fig:locality_preservation}
    % \vspace{-0.5cm}
\end{wrapfigure}

Next, we evaluate how well $Z$-order curve projections preserve locality across different dimensions and sample sizes. Specifically, we test the locality preservation by measuring the overlap between the top-64 nearest neighbors before and after projection, with sample sizes $N \in \{512, 1024, 2048\}$.

\autoref{fig:locality_preservation} shows the relationship between locality preservation and the dimensionality $d_K$. As $d_K$ increases, the overlap between the top-64 nearest neighbors diminishes for all sample sizes, indicating a decrease in locality preservation. Lower $d_K$ values exhibit a higher level of locality preservation across all sample sizes. However, for larger sample sizes, such as $N = 2048$, the drop in locality preservation is more pronounced as the dimensionality increases.  We select $d_K=3$ for \methodname.
% , and, with the multi-head mechanism, the probability of a token not being attended to by \methodname is very low (e.g., $(1-58.3\%)^8$ for 8 heads, $N=2048$).

These results highlight the importance of choosing an appropriate $d_K$ for maintaining locality, especially for larger datasets where higher dimensionality can lead to distortions in spatial proximity after a projection. 

\subsection{The Effects of $k$ in \methodname}
\label{sec:top-$k$}
As an ablation study, we explore the effect of varying $k$ in the attention mechanism of \methodname. The goal of this experiment is to analyze how different values of $k$ influence the model's performance on the associative recall task across different model dimensions.

\autoref{fig:top_k_ablation} shows \methodname~to achieve near-perfect accuracy across all model dimensions (32, 64, 128, and 256) for different values of $k$ ranging from 16 to 48. In most of our experiments, we set $k = 32$, as it provides a good balance between performance and computational efficiency. Interestingly, there is little variation in accuracy between different values of $k$, indicating that \methodname~is robust to changes in this parameter.

% The choice of $k$ presents a trade-off between maintaining prediction performance and the increased computational resources required. Overall, setting $k$ to 32 offers an optimal trade-off between accuracy and efficiency, and is thus used in the majority of our experiments. 
{
\color{black}

\subsection{Efficiency Benchmarking}\label{sec:efficiency}

In order to better understand the effectiveness of \methodname, we further conduct an experimental study to demonstrate its computational efficiency in comparison to existing attention methods. In particular, we compare with a naive attention implementation from \texttt{PyTorch} (based on \citet{vaswani2017attention}) as well as an IO-aware Flash-Attention ~\citep{dao2022flashattention, dao2023flashattention2}. Our implementation is based on Triton.
% and we conduct this benchmarking on a single NVIDIA A100 80GB GPU.

\begin{table}[h!]
    \centering
    \caption{Time (in milliseconds) for different operations to compute for a fixed-sized batch of varying sequence length. Our method outperforms a naive attention implementation across all lengths while also outperforming Flash-Attention by a signficant margin as the sequence length increases.}
    \resizebox{0.8\linewidth}{!}{
    \begin{tabular}{c|rr|rr|rr|rr}
    \toprule
    Method & \multicolumn{2}{c}{Torch Attention} & \multicolumn{2}{c}{Mamba} & \multicolumn{2}{c}{Flash Attention} & \multicolumn{2}{c}{ZETA} \\
    \midrule
    Input Length & \texttt{FWD} & \texttt{FWD+BWD} & \texttt{FWD} & \texttt{FWD+BWD}& \texttt{FWD} & \texttt{FWD+BWD} &\texttt{FWD} & \texttt{FWD+BWD}\\
    \midrule
4096 & 44.3 & 117.9 & 7.1 & 14.0 & 3.4 & 29.2 & 5.6 & 38.2 \\
8192 & OOM & OOM & 11.8 & 23.0 & 12.8 & 111.5 & 11.0 & 76.4 \\
16384 & OOM & OOM & 23.5 & 45.7 & 50.4 & 437.7 & 21.7 & 152.6 \\
32768 & OOM & OOM & 47.3 & 91.8 & 198.2 & 1733.5 & 43.0 & 304.8 \\
65536 & OOM & OOM & 94.0 & 183.7 & 805.3 & 7044.1 & 85.8 & 608.2 \\
\bottomrule
    \end{tabular}
    }
    \label{tab:time}
\end{table}
\vspace{0.25cm}
\begin{table}[h!]
    \centering
    \caption{Memory consumption (in MB) for different operations to compute for a fixed-sized batch of varying sequence length. Our method outperforms a naive attention implementation across all lengths while only marginally trailing a highly optimized Flash Attention implementation.}
    \resizebox{0.9\linewidth}{!}{
    \begin{tabular}{c|rr|rr|rr|rr}
    \toprule
    Method & \multicolumn{2}{c}{Torch Attention} & \multicolumn{2}{c}{Mamba} & \multicolumn{2}{c}{Flash Attention} & \multicolumn{2}{c}{ZETA} \\
    \midrule
    Input Length & \texttt{FWD} & \texttt{FWD+BWD} & \texttt{FWD} & \texttt{FWD+BWD}& \texttt{FWD} & \texttt{FWD+BWD} &\texttt{FWD} & \texttt{FWD+BWD}\\
    \midrule
4096 & 17268.1 & 25972.1 & 574.2& 632.2 & 886.1& 1784.1& 1314.1& 1926.1 \\
8192 & OOM  & OOM & 904.2 & 1020.2& 1528.1& 3324.1& 2382.1& 3606.1 \\
16384 & OOM  & OOM & 1564.2 & 1776.2 & 2812.1& 6404.1& 4520.1& 6968.2 \\
32768 & OOM  & OOM  & 2884.2 & 3200.2& 5380.1& 12564.1& 8796.2& 13692.2 \\
65536 & OOM  & OOM & 5524.2& 6048.2& 10516.1& 24884.1& 17348.2& 27140.3 \\
\bottomrule
    \end{tabular}
    }
    \label{tab:memory}
\end{table}
\vspace{1cm}

\autoref{tab:time} indicates the time required for both a forward pass as well as a forward-backward pass using our efficient \methodname~implementation as well as the aforementioned attention implementations. We observe that our implementation significantly outspeeds a naive implementation of attention and do not suffer from out of memory issues while also outperforming Flash-Attention for long sequences, with a widening gap as the sequence length increases. This indicates both the computational efficiency of our method as well as serves as an empirical validation of the $\mathcal{O}(N\log N)$ complexity of \methodname~which we previously justify theoretically. Furthermore, if we compare with Mamba~\citep{gu2024mamba}, we demonstrate that \methodname~has a faster forward pass while the forward-backward pass maintains a similar relative performance as the sequence length increases.

\autoref{tab:memory} meanwhile shows that \methodname~uses less memory than a naive attention implementation while also only slightly utilizing more memory than a highly optimized Flash Attention implementation. Nevertheless, in comparison to a sequence model such as Mamba, all attention models predictably use more memory due to the use of softmax-type operations.
}

\section{Conclusion}
In this paper, we presented \methodname, a model designed to enhance the efficiency of top-$k$ attention by leveraging $Z$-order curves for parallel token selection in one-dimensional space, reducing both time and space complexity to $\mathcal{O}(N \log N)$. By carefully selecting the dimensionality of key and query pairs, \methodname~effectively preserves relative distances, improving both locality and computational efficiency. Our comprehensive experiments on synthetic associative recall, \textsc{LRA}, and \textsc{WikiText-103} demonstrate that \methodname~consistently matches or outperforms traditional attention mechanisms, making it particularly well-suited for long-sequence tasks that demand scalability and efficiency. Additionally, the introduction of the Adaptive Cauchy-Softmax mechanism enhances \methodname's flexibility, enabling it to handle long-range dependencies more effectively and efficiently. Overall, $\methodname$ offers a robust, scalable, and efficient solution for sequence modeling, combining adaptive token selection with dynamic softmax to optimize performance across a range of tasks and datasets.

% In this paper, we presented \methodname, a model that enhances the efficiency of top-$k$ attention by leveraging $Z$-order curves for parallel token selection in one-dimensional space, reducing the time and space complexity to $\mathcal{O}(N \log N)$. By selecting the dimensionality of key and query pairs, \methodname~preserves relative distances while improving computational efficiency. Our experiments on Associative Recall, Long-Range Arena (LRA), and WikiText-103 demonstrated that \methodname~consistently matches or outperforms standard attention mechanisms, making it well-suited for long-sequence tasks. Additionally, the introduction of Adaptive Cauchy-Softmax further enhances \methodname's flexibility, enabling efficient handling of long-range dependencies. Overall, \methodname~offers a scalable and effective solution for efficient sequence modeling.
\newpage
\section*{Acknowledgements}
We appreciate constructive feedback from anonymous reviewers and meta-reviewers. This work is supported by the Natural Sciences and Engineering Research Council of Canada (NSERC), Discovery Grants program.
\bibliography{iclr2025_conference}
\bibliographystyle{iclr2025_conference}

\appendix
\newpage

\section{Theoretical Analysis}

The recent paper, Reformer, proposed the Shared-$QK$ Transformer that shares the linear projection layer for keys and queries. It reduces the number of parameters and memory space while not affecting the model performance.
https://arxiv.org/abs/2001.0445a very simple but efficient technique.
\begin{lemma}[Johnson–Lindenstrauss Lemma]
For any $0 < \epsilon < 1$ and any integer $m$, let $d$ be a positive integer such that $d = \Omega (\frac{ \ln m}{\epsilon^2})$. Then for any set $x$ of $m$ points in $\mathbb{R}^D$, there exists a map $f: \mathbb{R}^D \to \mathbb{R}^d$ such that for all $x_i, x_j \in \mathcal{X}$,
\begin{align}
(1 - \epsilon) \|x_i - x_j\|^2 \leq \|f(x_i) - f(x_j)\|^2 \leq (1 + \epsilon) \|x_i - x_j\|^2.
\end{align}
\end{lemma}

\begin{assumption}
    Let $\alpha \in \Delta_{m-1}$ be an element of the $m$-dimensional simplex, defined as $\Delta_{m-1} \triangleq \left\{\alpha \in \mathbb{R}^m \mid \alpha_i \geq 0, \sum_{i=1}^{m} \alpha_i = 1\right\}$. Assume that $h_{\mathrm{attn}}$ equipped with $\alpha$ can achieve an optimal learnable similarity function $\Gamma$, where the attention scores are given by $\alpha = \mathrm{softmax}(\Gamma(f_k(x_i), f_q(x)))$, such that $\Gamma$ is trained to be optimal to have the minimal expected risk: $ \min_{\alpha}\|h_{\mathrm{attn}}(x,S_x;\alpha) - y\|$, where $h_{\mathrm{attn}}$ denotes the attention-based hypothesis, $x$ is the input, and $S_x$ is the context.
\end{assumption}

\begin{lemma}\label{lemma:monotone}
    Let $\alpha\in \mathbb{R}^m$, $\forall h_1:\mathbb{R}^m\rightarrow \mathbb{R}$ and $\forall h_2:\mathbb{R}^m\rightarrow \mathbb{R}$. Assume that $ h_1(\alpha)\le h_2(\alpha)$. Then $\min_{\alpha} h_1(\alpha)\le \min_{\alpha} h_2(\alpha)$.
\end{lemma}
\begin{proof}
    Let $\alpha_1 = \arg\min_\alpha h_1(\alpha)$ and $\alpha_2 = \arg\min_\alpha h_2(\alpha)$. Then, $h_1(\alpha_2) \le h_2(\alpha_2) $ due to condition $h_1(\alpha) \le h_2(\alpha)$. We also have $h_1(\alpha_1)\le h_1(\alpha_2) $ due to condition $\alpha_1$ achieving the minimum of $h_1$. To sum up, we have $h_1(\alpha_1) \le h_2(\alpha_2) $. 
    
\end{proof}

\begin{lemma}\label{lemma:sets-intersect}~\citep{shalev2014understanding}
Let $C_1, \ldots, C_r$ be a collection of subsets of some domain set, $\mathcal{X}$. Let $S$ be a sequence of $m$ points sampled i.i.d. according to some probability distribution, $\mathcal{D}$, over $\mathcal{X}$. Then,
$$
\mathbb{E}_{S \sim \mathcal{D}^m} \left[ \sum_{i: C_i \cap S = \emptyset} \mathbb{P}[C_i] \right] \leq \frac{r}{m e}.
$$
\end{lemma}
\begin{proof} \quad From the linearity of expectation, we can rewrite
\begin{align*}
\mathbb{E}_S \left[\sum_{i: C_i \cap S = \emptyset} \mathbb{P}[C_i] \right] &= \sum_{i=1}^{r} \mathbb{P}[C_i] \mathbb{E}_S \left[\mathbf{1}_{C_i \cap S = \emptyset}\right].
\end{align*}

Next, for each $i$ we have
\begin{align*}
\mathbb{E}_S \left[\mathbf{1}_{C_i \cap S = \emptyset}\right] &= \mathbb{P}_S[C_i \cap S = \emptyset] = (1 - \mathbb{P}[C_i])^m \leq e^{-\mathbb{P}[C_i] m}.
\end{align*}

Combining the preceding two equations, we get
\begin{align*}
\mathbb{E}_S \left[\sum_{i: C_i \cap S = \emptyset} \mathbb{P}[C_i] \right] &\leq \sum_{i=1}^{r} \mathbb{P}[C_i] e^{-\mathbb{P}[C_i] m} \leq r \max_i \mathbb{P}[C_i] e^{-\mathbb{P}[C_i] m}.
\end{align*}

Finally, by elementary calculus, $\underset{a}{\max} \ ae^{-ma} \leq \frac{1}{me}$, concluding the proof.
\end{proof}

\begin{theorem}
Let $\mathcal{X} \in \mathbb{R}^d$, $\mathcal{Y} \in \mathbb{R}^D$, and $\mathcal{D}$ be a distribution over $\mathcal{X} \times \mathcal{Y}$ for which the conditional probability function, $h: \mathbb{R}^d \rightarrow \mathbb{R}^{D}$, is a $l$-Lipschitz function. Let $h$ denote a hypothesis, and $h_{\mathrm{attn}}$ denote the one-layer attention model to aggregate the predictions of a sample set $S \sim \mathcal{D}^m$ to predict for another i.i.d sample $x$. Specifically, here we assume the same linear map for a key mapping $f_k$ and a query mapping $f_q$ as $f:\mathbb{R}^d \rightarrow [-B,B]^{d_K}$ where $B$ is the bound of projection of $x$ by $f$, and we assume the value mapping $f_v$ be the Bayes optimal rule $h^*=\mathbb{E}[Y|X=x]$, which is the hypothesis that minimizes $L_{\mathcal{D}}(h)$ over all functions. Then, the expected risk of the regression tasks $L_{\mathcal{D}}(h) = \mathbb{E}_{(x,y) \sim \mathcal{D}}\|h(x) - y\|$ with $h$ as $h_{\mathrm{attn}}$ can be upper bounded by
$$\mathbb{E}_{S \sim \mathcal{D}^m} \left[L_{\mathcal{D}}(h_{\mathrm{attn}})\right] \leq L_{\mathcal{D}}(h^*) +\frac{4lc \sqrt{d_K}B  m^{-1/(d_K+1)}}{\sqrt{1-\sqrt{\frac{C\ln m}{d_K}}}}.$$
\end{theorem}
\begin{proof}
$\mathbb{E}_S[L_{\mathcal{D}}(h)]$ is the root mean square error (RMSE) between the prediction and $y$ conditioned on a sampled set $S$ and an additional example $(x, y)$, such that the label of $\pi_1(x)$ is different from $y$. In other words, we first sample $m$ examples, $S_x = \{{x}_1, \dots, x_m\}$, according to $\mathcal{D}_\mathcal{X}$, and an additional example, $(x,y)$. It follows that
\begin{align*}
  \mathbb{E}_S[L_{\mathcal{D}}(h_{\mathrm{attn}})] &= \mathbb{E}_{S_x \sim \mathcal{D}_\mathcal{X}^m, x \sim \mathcal{D}_\mathcal{X}} \mathbb{E}_{y \sim \mathcal{D}_{\mathcal{Y}|\mathcal{X}}} \min_{\alpha}\|h_{\mathrm{attn}}(x,S_x;\alpha) - y\|  \\ 
  &= \mathbb{E}_{S_x \sim \mathcal{D}_\mathcal{X}^m, x \sim \mathcal{D}_\mathcal{X}} \mathbb{E}_{y \sim \mathcal{D}_{\mathcal{Y}|\mathcal{X}}} \min_{\alpha}\|\sum_{i=1}^m \alpha_i h^*(x_i) - y\|  \\
  &\le \mathbb{E}_{S_x \sim \mathcal{D}_\mathcal{X}^m, x \sim \mathcal{D}_\mathcal{X}} \mathbb{E}_{y \sim \mathcal{D}_{\mathcal{Y}|\mathcal{X}}} \left[ \min_{\alpha}\|\sum_{i=1}^m \alpha_i h^*(x_i) - h^*(x)\| + \| h^*(x) - y\| \right]\\
  &= L_{\mathcal{D}}(h^*) + \mathbb{E}_{S_x \sim \mathcal{D}_\mathcal{X}^m, x \sim \mathcal{D}_\mathcal{X}} \mathbb{E}_{y \sim \mathcal{D}_{\mathcal{Y}|\mathcal{X}}}  \min_{\alpha}\|\sum_{i=1}^m \alpha_i h^*(x_i) - h^*(x)\|,
\end{align*}
where $\alpha_i$ is the attention score between $x_i$ and $x$. The inequality follows the Cauchy-Schwarz inequality and Lemma~\ref{lemma:monotone}.
\begin{align*}
    &\mathbb{E}_{S_x \sim \mathcal{D}_\mathcal{X}^m, x \sim \mathcal{D}_\mathcal{X}} \mathbb{E}_{y \sim \mathcal{D}_{\mathcal{Y}|\mathcal{X}}}  \min_{\alpha}\|\sum_{i=1}^m \alpha_i h^*(x_i) - h^*(x)\| \le  \mathbb{E}_{S_x \sim \mathcal{D}_\mathcal{X}^m, x \sim \mathcal{D}_\mathcal{X}} \mathbb{E}_{y \sim \mathcal{D}_{\mathcal{Y}|\mathcal{X}}}  \min_{\alpha}\sum_{i=1}^m \alpha_i \| h^*(x_i) - h^*(x)\|\\
    &\le \mathbb{E}_{S_x \sim \mathcal{D}_\mathcal{X}^m, x \sim \mathcal{D}_\mathcal{X}} \mathbb{E}_{y \sim \mathcal{D}_{\mathcal{Y}|\mathcal{X}}} \left[ l\cdot \min_{\alpha} \sum_{i=1}^m \alpha_i\| x_i - x\|\right]
     \le \mathbb{E}_{S_x \sim \mathcal{D}_\mathcal{X}^m, x \sim \mathcal{D}_\mathcal{X}} \mathbb{E}_{y \sim \mathcal{D}_{\mathcal{Y}|\mathcal{X}}} \left[ l\cdot \min_{\alpha}\sum_{i=1}^m \alpha_i\frac{\| f(x_i) - f(x)\|}{\sqrt{1-\sqrt{\frac{C\ln m}{d_K}}}}\right]\\
    &= \mathbb{E}_{S_x \sim \mathcal{D}_\mathcal{X}^m, x \sim \mathcal{D}_\mathcal{X}} \mathbb{E}_{y \sim \mathcal{D}_{\mathcal{Y}|\mathcal{X}}} \left[ l\cdot \min_{\alpha} \sum_{i=1}^m \alpha_i\frac{\|k_i - q\|}{\sqrt{1-\sqrt{\frac{C\ln m}{d_K}}}}\right]
\end{align*}
The first inequality follows from the Cauchy-Schwarz inequality; the second inequality follows from the $l$-Lipschitzness of $h^*$; the third inequality follows from the Johnson–Lindenstrauss Lemma, where we define $\epsilon$ in the Johnson–Lindenstrauss Lemma as $\epsilon = \sqrt{\frac{C\ln m}{d_K}}$, where $C$ is a constant; $k_i=f(x_i)$ and $q=f(x)$. It is obvious that $$ \min_{\alpha} \sum_{i=1}^m \alpha_i\|k_i - q\|=\|k_{\pi_1(q)} - q\|$$
where $k_{\pi_1(q)}$ is the closest $k_i$ to $q$. Thus we have
\begin{align*}
    &\mathbb{E}_{S_x \sim \mathcal{D}_\mathcal{X}^m, x \sim \mathcal{D}_\mathcal{X}} \mathbb{E}_{y \sim \mathcal{D}_{\mathcal{Y}|\mathcal{X}}}  \min_{\alpha}\left\|\sum_{i=1}^m \alpha_i h^*(x_i) - h^*(x)\right\| = \mathbb{E}_{S_x \sim \mathcal{D}_\mathcal{X}^m, x \sim \mathcal{D}_\mathcal{X}} \mathbb{E}_{y \sim \mathcal{D}_{\mathcal{Y}|\mathcal{X}}} \left[ l\cdot \frac{\|k_{\pi_1(q)} - q\|}{\sqrt{1-\sqrt{\frac{C\ln m}{d_K}}}}.\right]
\end{align*}
Fix some $\zeta = 2B/T$. For some integer $T$, let $r = T^{d_K}$ and $C_1, \dots, C_r$ be the cover of the set $\mathcal{X}$ using boxes of length $\zeta$. Namely, for every $(\xi_1, \dots, \xi_d) \in [T]^d$, there exists a set $C_i$ of the form $\{ k : \forall j, k_j \in [2B(\xi_j - 1)/T, 2B\xi_j / T]\}$. 

For each $k, q$ in the same box we have $\|k - q\| \leq \sqrt{d_K}B \, \zeta$. Otherwise, $\|k - q\| \leq \sqrt{d_K}B$. Therefore,

\begin{align*}
\mathbb{E}_{S_x, x} \left[\left\|k_{\pi_1(q)} - q\right\|\right] &\leq \mathbb{E}_S \left[\mathbb{P}\left[\bigcup_{i:C_i \cap S_x = \emptyset} C_i\right] \sqrt{d_K}B + \mathbb{P}\left[\bigcup_{i:C_i \cap S_x \neq \emptyset} C_i\right] B\zeta \sqrt{d_K}\right],
\end{align*}
and by combining Lemma \ref{lemma:sets-intersect} with the trivial bound 
\begin{align*}
\mathbb{P}[\bigcup_{i:C_i \cap S_x \neq \emptyset} C_i] \leq 1,
\end{align*}
we get that
\begin{align*}
\mathbb{E}_{\mathbf{x}, S_x} \left[\| k_{\pi_1(q)}-q\|\right] &\leq \sqrt{d_K} B\left(\frac{r}{me} + \zeta\right).
\end{align*}
Since the number of boxes is $r = (1/\zeta)^{d_K}$, it follows that
\begin{align*}
\mathbb{E}_{S_x, \mathbf{x}} \left[\left\| k_{\pi_1(q)}-q\right\|\right] &\leq \sqrt{d_K} B\left(\frac{2^{d_K} \zeta^{-d_K}}{me} + \zeta\right).
\end{align*}
Setting $\zeta = 2 m^{-1/(d_K+1)}$ and noting that
\begin{align*}
\frac{2^{d_K} \zeta^{-d}}{me} + \zeta &= \frac{2^{d_K} 2^{-d_K} m^{d_K/\left(d_K+1\right)}}{me} + 2 m^{-1/\left(d_K+1\right)} \\
&= m^{-1/\left(d_K+1\right)}\left(\frac{1}{e} + 2\right) \leq 4 m^{-1/(d_K+1)},
\end{align*}
then combining the preceding with previous results, we obtain
\begin{align*}
\mathbb{E}_S \left[L_{\mathcal{D}}(h_{\mathrm{attn}})\right] &\leq L_{\mathcal{D}}(h^\star) + \frac{4lc \sqrt{d_K}B  m^{-1/(d_K+1)}}{\sqrt{1-\sqrt{\frac{C\ln m}{d_K}}}}.
\end{align*}
\end{proof}

{\color{black}
\begin{figure}[t!]
\centering
\includegraphics[width=10cm]{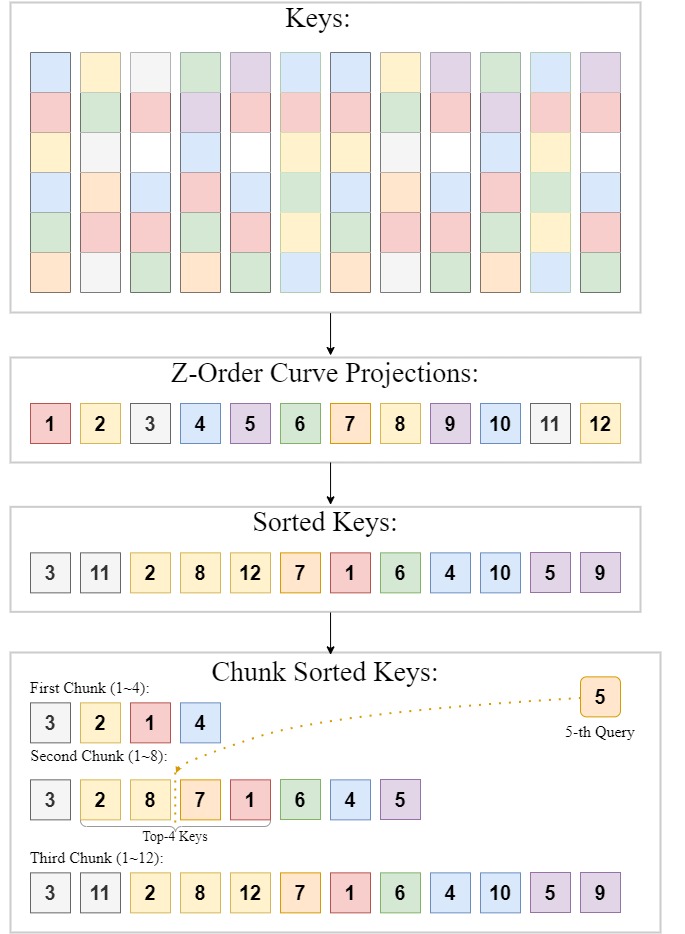}
\caption{Illustration of the chunking process in \methodname: Keys are projected into one-dimensional space using $Z$-order curves, sorted, and partitioned into chunks for efficient retrieval of top-$k$ keys for each query.}
\end{figure}

\section{An Illustrative Example of Efficient Top-$k$ Selection via $Z$-order Curve Chunking}

In \methodname, the chunking process is a crucial step for efficient key-query matching in large-scale attention mechanisms. The process begins by projecting the high-dimensional keys into a one-dimensional space using $Z$-order curves, which preserve spatial locality. These projections create a linear representation of the keys, as shown in the second row of the figure.

After projection, the keys are radix sorted in $\mathcal{O}(N)$ into ascending order of their $Z$-order integer values, enabling efficient binary searching in $\mathcal{O}(N)$ for the most relevant keys in the one-dimensional space. Sorting ensures that keys that are spatially close in the original multi-dimensional space remain close in the projected space, making the retrieval process computationally efficient.

Next, the sorted keys are partitioned into chunks. Each chunk contains a fixed number of keys, and queries are matched with the keys within their respective chunks. This chunk-based structure facilitates parallel processing, where each query can efficiently search for its top-$k$ nearest keys within a local subset of the sorted key space. For instance, in the figure, the 5th query retrieves its top-$4$ keys from the chunks containing its most relevant keys.

This chunking approach significantly reduces the computational overhead compared to searching the entire key space for each query, while leveraging the locality-preserving properties of $Z$-order curves. The result is a scalable and efficient mechanism for top-$k$ selection in \methodname, ensuring both speed and accuracy for attention-based tasks.
\begin{algorithm}[t]

\caption{ZETA Top-$k$ Attention Using $Z$-order Curves and Chunking}
\SetAlgoLined
\KwData{Keys $\mathcal{K}$ of size $M \times d$, sequence length $N$, chunk size $M$, query index $i$, top-$k$ value $k$}
\KwResult{Top-$k$ nearest keys for each query}
\BlankLine
\textbf{Step 1: $Z$-order Curve Projection}\;
\For{each key vector in $\mathcal{K}$}{
    Project each $d$-dimensional key vector into a one-dimensional representation using $Z$-order curve projection\;
}

\textbf{Step 2: Sorting of Keys}\;
Sort the projected one-dimensional keys in ascending order\;

\textbf{Step 3: Chunk Division}\;
Divide the sorted keys into multiple chunks based on the chunk size $M$\;

\textbf{Step 4: Chunk-wise Causal Masking}\;
\For{each query at position $i$}{
    Determine chunk index $m = \left\lfloor \frac{i}{M} \right\rfloor$\;
    Search only in the first $m$ chunks for the top-$k$ keys to attend to\;
    Exclude keys from positions $j > m \times M$\;
}

\textbf{Step 5: Nearest Neighbor Search}\;
\For{each query}{
    Perform nearest neighbor search within the selected chunks to find top-$k$ tokens\;
    Collect the indices of the nearest neighbors $I_q$ while maintaining causal constraints\;
}

\Return{Top-$k$ nearest keys for each query}\label{algo:chunk}
\end{algorithm}

The pseudo-code in Algorithm~\ref{algo:chunk} outlines the ZETA Top-$k$ Attention mechanism, which combines $Z$-order curve projections with chunk-based sorting to efficiently identify and retrieve the top-$k$ nearest neighbors while maintaining causal constraints. It provides a structured approach to reduce computational overhead by limiting the search space to relevant chunks, ensuring both efficiency and adherence to sequence masking requirements.

\section{Experiments Details}
The ZETA model configuration generally involves setting the number of chunks to values such as ${4, 8, 16, 32}$ depending on the sequence length, which provides a flexible way to handle different input scales effectively. This chunking strategy facilitates parallelism in processing, allowing for efficient memory use and computational speedup during attention operations. The hidden dimension, $d_V$, is typically set to 256 or 512 with 8 attention heads when working with \textsc{LRA} datasets. However, for larger and more complex datasets such as \textsc{WikiText-103}, the hidden dimension is increased to $d_V = 768$ with 12 attention heads to ensure that the model has sufficient capacity to learn intricate long-range dependencies effectively. Additionally, the dimensions of keys and queries are kept significantly lower at $d_K = d_Q = 3$, which aids in reducing the computational burden and mitigates the ``curse of dimensionality" while still preserving enough information for efficient attention computation. This choice of dimensions strikes a balance between model efficiency and effectiveness, making ZETA well-suited for long-sequence modeling tasks.

\section{I/O-Aware ZETA Optimized with Triton}

Our Triton implementation of \methodname~focuses on improving the efficiency of sparse attention through customized kernel programming. We leverage Triton to create specialized GPU kernels for top-$k$ sparse attention. The \texttt{sparse\_topk\_attention\_kernel} and its corresponding backward pass kernel \texttt{sparse\_topk\_attention\_backward\_kernel} are implemented using the Triton JIT (Just-In-Time) compiler. This approach allows for significant speedup by optimizing memory access patterns and reducing I/O overhead during the computation. The kernel is tuned to different configurations, like ``block\_size" and ``num\_warps", which directly influence how GPU resources are allocated. Especially, we compute the mean vector of history tokens in the block of the current kernel, instead of computing global mean vectors, which effectively reduce overheads. he \texttt{@triton.autotune} decorator is used to evaluate multiple kernel configurations for optimal performance, making sure that GPU resources are well-utilized depending on sequence length and other parameters.

One key challenge addressed is efficient indexing for large tensors in the backward and forward computations. In Triton, indexing is achieved via program IDs, \texttt{tl.program\_id()}, that are used to identify which part of the workload is being computed by each block or thread, ensuring that parallelism is effectively exploited. The Triton kernel computes Cauchy Softmax for top-$k$ KV pairs for each query, employing a specialized kernel to access only the most relevant $k$ keys during the attention process. This reduces the computational complexity compared to a full attention mechanism, and Triton’s low-level bit manipulation operations (\texttt{tl.load} and \texttt{tl.store}) are used for fast data retrieval.

The sparse Attention Mechanism is computed by considering only the top-$k$ keys per query, which significantly reduces the computational load. The \texttt{sparse\_topk\_attention\_kernel} involves efficient gathering of keys and values based on top-$k$ indices. The indices are pre-computed in a sorted order, which enables efficient retrieval without scanning the entire key space.

The Triton kernel also includes custom backward functions to handle the gradient flow. The backward kernel \texttt{sparse\_topk\_attention\_backward\_kernel} computes gradients for each parameter involved in the sparse attention operation, including $q, k, v$ and the learnable parameter $\gamma$. Triton’s tl.atomic\_add is used to accumulate gradients, ensuring that all updates to shared memory locations are synchronized.

During the forward pass, intermediate values such as the Euclidean distances and normalization constants are stored. These values are reused in the backward pass, which reduces redundant computations and accelerates the training process.

By using Triton, we managed to reduce the I/O overhead that traditional \texttt{PyTorch} operations faced, especially during backward computations. Furthermore, we used fused kernels to mitigate the overhead associated with multiple indexing operations. This fusion helps in reducing the number of kernel launches, which translates to reduced latency and faster execution, as Triton allows more control over memory coalescing and efficient block-wise operations.

Overall, the Triton-based implementation in ZETA allows for a more scalable sparse attention mechanism that retains the benefits of locality preservation through $Z$-order Curves while significantly reducing computational and I/O bottlenecks. This makes the ZETA attention more suitable for long-sequence tasks where traditional transformers are too resource-intensive.

\section{Gradient Derivation for the Backward Pass of Triton ZETA}

\subsection{Updated Attention Mechanism with Sparse Attention}

We introduce sparse attention by computing values and keys from an index set \( I_i \) of top-$k$ tokens specific to each query \( \mathbf{q}_i \). The unnormalized attention scores are defined as:

\begin{equation}
S_{ij} = \frac{1}{D_{ij} + \varepsilon}, \quad \mathrm{for} \ j \in I_i
\end{equation}

where:

\begin{itemize}
    \item \( D_{ij} = \| \mathbf{q}_i - \mathbf{k}_j \|^2 \)
    \item \( \varepsilon = \gamma^2\)  is a trainable scalar parameter
\end{itemize}

Define:

\begin{equation}
\delta_{ij} = D_{ij} + \varepsilon
\end{equation}

The steps of the attention mechanism are:
\begin{enumerate}
    \item \textbf{Compute Pairwise Distances}:
   \begin{equation}
   D_{ij} = \| \mathbf{q}_i - \mathbf{k}_j \|^2, \quad \mathrm{for} \ j \in I_i
   \end{equation}
   \item \textbf{Compute Unnormalized Attention Scores}:
   \begin{equation}
   S_{ij} = \frac{1}{\delta_{ij}}, \quad \mathrm{for} \ j \in I_i
   \end{equation}
   \item \textbf{Normalize Attention Weights}:
   \begin{align}
   Z_i &= \sum_{j \in I_i} S_{ij} \\
   A_{ij} &= \frac{S_{ij}}{Z_i}, \quad \mathrm{for} \ j \in I_i
   \end{align}
   \item \textbf{Compute Output}:
   \begin{equation}
   \mathbf{o}_i = \sum_{j \in I_i} A_{ij} \mathbf{v}_j
   \end{equation}
\end{enumerate}

Our goal is to compute the gradients of the output \( \mathbf{o}_i \) with respect to \( \mathbf{q}_i \), \( \mathbf{k}_j \), \( \mathbf{v}_j \), and \( \varepsilon \), considering the sparse attention.

\subsection{Gradients with Respect to \( \mathbf{Q} \), \( \mathbf{K} \), \( \mathbf{V} \), and \( \varepsilon \)}

The gradient of the output \( \mathbf{o}_i \) with respect to \( \mathbf{v}_j \) is:

\begin{equation}
\frac{\partial \mathbf{o}_i}{\partial \mathbf{v}_j} =
\begin{cases}
A_{ij}, & \text{if } j \in I_i \\
0, & \text{otherwise}
\end{cases}
\end{equation}

% \subsubsection*{1. Compute Partial Derivatives}

\textbf{Derivative of \( D_{ij} \) with respect to \( \mathbf{q}_i \) and \( \mathbf{k}_j \):}

\begin{align}
\frac{\partial D_{ij}}{\partial \mathbf{q}_i} &= 2 (\mathbf{q}_i - \mathbf{k}_j), \quad \mathrm{for} \ j \in I_i \\
\frac{\partial D_{ij}}{\partial \mathbf{k}_j} &=
\begin{cases}
-2 (\mathbf{q}_i - \mathbf{k}_j), & \text{if } j \in I_i \\
0, & \text{otherwise}
\end{cases}
\end{align}

\textbf{Derivative of \( S_{ij} \):}

First, compute the derivative of \( S_{ij} \) with respect to \( \delta_{ij} \):

\begin{equation}
\frac{\partial S_{ij}}{\partial \delta_{ij}} = -\frac{1}{\delta_{ij}^2}, \quad \mathrm{for} \ j \in I_i
\end{equation}

Compute the derivative of \( \delta_{ij} \) with respect to \( D_{ij} \) and \( \varepsilon \):

\begin{align}
\frac{\partial \delta_{ij}}{\partial D_{ij}} &= 1, \quad \mathrm{for} \ j \in I_i \\
\frac{\partial \delta_{ij}}{\partial \varepsilon} &= 1, \quad \mathrm{for} \ j \in I_i
\end{align}

Now, compute the derivative of \( S_{ij} \) with respect to \( D_{ij} \) and \( \varepsilon \):

\begin{align}
\frac{\partial S_{ij}}{\partial D_{ij}} &= -\frac{1}{\delta_{ij}^2}, \quad \mathrm{for} \ j \in I_i \\
\frac{\partial S_{ij}}{\partial \varepsilon} &= -\frac{1}{\delta_{ij}^2}, \quad \mathrm{for} \ j \in I_i
\end{align}

\textbf{Derivative of \( S_{ij} \) with respect to \( \mathbf{q}_i \) and \( \mathbf{k}_j \):}

\begin{align}
\frac{\partial S_{ij}}{\partial \mathbf{q}_i} &= -\frac{2 (\mathbf{q}_i - \mathbf{k}_j)}{\delta_{ij}^2}, \quad \mathrm{for} \ j \in I_i \\
\frac{\partial S_{ij}}{\partial \mathbf{k}_j} &=
\begin{cases}
\frac{2 (\mathbf{q}_i - \mathbf{k}_j)}{\delta_{ij}^2}, & \text{if } j \in I_i \\
0, & \text{otherwise}
\end{cases}
\end{align}

% \subsubsection*{2. Derivative of Normalized Weights \( A_{ij} \)}

\textbf{Derivative of \( Z_i \) with respect to \( S_{ij} \):}

\begin{equation}
\frac{\partial Z_i}{\partial S_{ij}} = 1, \quad \mathrm{for} \ j \in I_i
\end{equation}

\textbf{Derivative of \( A_{ij} \) with respect to \( S_{ij} \):}

\begin{equation}
\frac{\partial A_{ij}}{\partial S_{ij}} = \frac{Z_i - S_{ij}}{Z_i^2}, \quad \mathrm{for} \ j \in I_i
\end{equation}

\textbf{Derivative of \( A_{ij} \) with respect to \( S_{il} \) for \( l \neq j \):}

\begin{equation}
\frac{\partial A_{ij}}{\partial S_{il}} =
\begin{cases}
-\dfrac{S_{ij}}{Z_i^2}, & \text{if } l \in I_i \\
0, & \text{otherwise}
\end{cases}
\end{equation}

% \subsubsection{3. Compute Gradient of Output \( \mathbf{o}_i \)}

\textbf{Gradient of \( \mathbf{o}_i \) with respect to \( S_{il} \):}

\begin{align}
\frac{\partial \mathbf{o}_i}{\partial S_{il}} &= \sum_{j \in I_i} \mathbf{v}_j \frac{\partial A_{ij}}{\partial S_{il}} \\
&= \frac{\mathbf{v}_l - \mathbf{o}_i}{Z_i}, \quad \mathrm{for} \ l \in I_i
\end{align}

% \subsection{Compute Gradients with Respect to Parameters}

\textbf{Gradient with Respect to \( \mathbf{q}_i \)}:

\begin{align}
\frac{\partial \mathbf{o}_i}{\partial \mathbf{q}_i} &= \sum_{l \in I_i} \frac{\partial \mathbf{o}_i}{\partial S_{il}} \cdot \frac{\partial S_{il}}{\partial \mathbf{q}_i} \\
&= -2 \sum_{l \in I_i} \frac{\mathbf{v}_l - \mathbf{o}_i}{Z_i} \cdot \frac{(\mathbf{q}_i - \mathbf{k}_l)}{\delta_{il}^2}
\end{align}

\textbf{Gradient with Respect to \( \mathbf{k}_j \)}:

\begin{equation}
\frac{\partial \mathbf{o}_i}{\partial \mathbf{k}_j} =
\begin{cases}
2 \frac{\mathbf{v}_j - \mathbf{o}_i}{Z_i} \cdot \frac{(\mathbf{q}_i - \mathbf{k}_j)}{\delta_{ij}^2}, & \text{if } j \in I_i \\
0, & \text{otherwise}
\end{cases}
\end{equation}

The total gradient with respect to \( \mathbf{k}_j \) is:

\begin{equation}
\frac{\partial L}{\partial \mathbf{k}_j} = \sum_{i: j \in I_i} \left( 2 \frac{\partial L}{\partial \mathbf{o}_i} \cdot \frac{\mathbf{v}_j - \mathbf{o}_i}{Z_i} \cdot \frac{(\mathbf{q}_i - \mathbf{k}_j)}{\delta_{ij}^2} \right)
\end{equation}

\textbf{Gradient with Respect to \( \varepsilon \)}:

\begin{align}
\frac{\partial \mathbf{o}_i}{\partial \varepsilon} &= \sum_{l \in I_i} \frac{\partial \mathbf{o}_i}{\partial S_{il}} \cdot \frac{\partial S_{il}}{\partial \varepsilon} \\
&= -\sum_{l \in I_i} \frac{\mathbf{v}_l - \mathbf{o}_i}{Z_i} \cdot \frac{1}{\delta_{il}^2}
\end{align}

The total gradient with respect to \( \varepsilon \) is:

\begin{equation}
\frac{\partial L}{\partial \varepsilon} = -\sum_i \left( \frac{\partial L}{\partial \mathbf{o}_i} \sum_{l \in I_i} \frac{\mathbf{v}_l - \mathbf{o}_i}{Z_i} \cdot \frac{1}{\delta_{il}^2} \right)
\end{equation}

\subsection{Summary of Gradient Computations}

\begin{enumerate}
    \item \textbf{Compute \( D_{ij} \), \( \delta_{ij} \), \( S_{ij} \), \( Z_i \), \( A_{ij} \), and \( \mathbf{o}_i \):}

   \begin{align}
   D_{ij} &= \| \mathbf{q}_i - \mathbf{k}_j \|^2, \quad \mathrm{for} \ j \in I_i \\
   \delta_{ij} &= D_{ij} + \varepsilon, \quad \mathrm{for} \ j \in I_i \\
   S_{ij} &= \frac{1}{\delta_{ij}}, \quad \mathrm{for} \ j \in I_i \\
   Z_i &= \sum_{j \in I_i} S_{ij} \\
   A_{ij} &= \frac{S_{ij}}{Z_i}, \quad \mathrm{for} \ j \in I_i \\
   \mathbf{o}_i &= \sum_{j \in I_i} A_{ij} \mathbf{v}_j
   \end{align}

    \item \textbf{Compute Gradient with Respect to \( \mathbf{V} \):}

   \begin{equation}
   \frac{\partial L}{\partial \mathbf{v}_j} = \sum_{i: j \in I_i} \frac{\partial L}{\partial \mathbf{o}_i} \cdot A_{ij}
   \end{equation}
   
   \item \textbf{Compute Gradient with Respect to \( \mathbf{Q} \):}

   \begin{equation}
   \frac{\partial L}{\partial \mathbf{q}_i} = -2 \left( \frac{\partial L}{\partial \mathbf{o}_i} \right) \sum_{j \in I_i} \frac{\mathbf{v}_j - \mathbf{o}_i}{Z_i} \cdot \frac{(\mathbf{q}_i - \mathbf{k}_j)}{\delta_{ij}^2}
   \end{equation}
   
   \item 
   \textbf{Compute Gradient with Respect to \( \mathbf{K} \):}

   \begin{equation}
   \frac{\partial L}{\partial \mathbf{k}_j} = \sum_{i: j \in I_i} \left( 2 \frac{\partial L}{\partial \mathbf{o}_i} \cdot \frac{\mathbf{v}_j - \mathbf{o}_i}{Z_i} \cdot \frac{(\mathbf{q}_i - \mathbf{k}_j)}{\delta_{ij}^2} \right)
   \end{equation}
   
   \item 
   
   \textbf{Compute Gradient with Respect to \( \varepsilon \):}

   \begin{equation}
   \frac{\partial L}{\partial \varepsilon} = -\sum_i \left( \frac{\partial L}{\partial \mathbf{o}_i} \sum_{j \in I_i} \frac{\mathbf{v}_j - \mathbf{o}_i}{Z_i} \cdot \frac{1}{\delta_{ij}^2} \right)
   \end{equation}
\end{enumerate}

\subsection{Implementation Notes}

\begin{itemize}
    \item \textbf{Sparse Attention Index Set \( I_i \):}
    \begin{itemize}
        \item \( I_i \) is the set of indices that query \( \mathbf{q}_i \) attends to, which is collected from $Z$-order Curve projected sequences.
        \item The attention computations and gradient updates are performed only over \( j \in I_i \).
    \end{itemize}
    
    \item \textbf{Numerical Stability:}
    \begin{itemize}
        \item The addition of \( \varepsilon \) ensures that \( \delta_{ij} > 0 \) if \( \varepsilon > 0 \), preventing division by zero.
        \item Ensure that \( \varepsilon \) remains positive during training, possibly by parameterizing \( \varepsilon = \exp(\theta) \) where \( \theta \) is unconstrained.
    \end{itemize}
    
    \item \textbf{Efficient Computation:}
    \begin{itemize}
        \item Utilize sparse matrix representations to handle the index sets \( I_i \) efficiently.
        \item Use vectorized operations and appropriate masking to perform computations only over the valid indices.
    \end{itemize}
\end{itemize}

\section{Limitations}
Given that our method is a top-$k$ attention mechanism, there are some shared limitations between our method and that of prior work that deals with attention, such as there still potentially being higher chances to ignore attention to important information (with low attention scores) than full attention methods given the use of only the top-$k$ tokens.
\section{Additional Experiments}
\subsection{Ablation on Attention Performance with Varying $d_K$ on LRA}

We expand on the ablation study presented in Figure 2(b), focusing on the ListOps and Image tasks from the Long Range Arena (LRA) benchmark. Specifically, we examine the impact of varying the dimensionality of the keys and queries ($d_K$) on attention performance. The results are summarized in \autoref{tab:dk-lra}.

Our experimental findings indicate that the performance remains relatively consistent for $d_K \geq 3$, whereas a noticeable decline is observed for $d_K < 3$. This supports our hypothesis that, unlike value vectors, the keys and queries predominantly encode positional information rather than intricate semantic features. As such, reducing $d_K$ to a small value allows us to effectively lower computational costs without incurring a significant loss in performance. This insight guided our selection of $d_K$ values in subsequent experiments.
\begin{table}[h]
\centering
\begin{tabular}{c|cccccc}
\toprule
$d_K$     & 1    & 2    & 3    & 32   & 64   & 128  \\ \midrule
ListOps   & 31.04 & 34.79 & 36.06 & 36.19 & 36.32 & 36.37 \\
Image     & 37.6  & 40.23 & 42.64 & 42.72 & 42.51 & 42.44 \\ \bottomrule
\end{tabular}
\caption{Performance metrics for different values of $d_K$.}
\label{tab:dk-lra}
\end{table}

\subsection{Performance of \methodname\,using different similarity metric}

We utilize Euclidean distance for k-nearest neighbor (k-NN) searches to identify the top-$k$ attended tokens, as it is particularly well-suited for this purpose. In contrast, dot-product similarity cannot be directly employed for k-NN searches without normalization, as highlighted in \citep{mao2024iceformer}. Our experimental results also indicate that normalized dot-product similarity performs worse than Euclidean distance, as demonstrated in additional MQAR experiments below.

Specifically, Euclidean distance-based methods, such as Negative Euclidean with traditional softmax and Cauchy Softmax, consistently outperform dot-product-based methods for top-$k$ attention when using a small dimensionality ($d_k \leq 4$), which is the setting adopted in ZETA.

\begin{table}[h]
\centering
\begin{tabular}{c|cccc}
\toprule
$d_k$                & 1    & 2    & 3    & 4    \\ \midrule
Negative Euclidean   & 64.6 & 99.4 & 99.4 & 99.3 \\
Inverse Euclidean    & 22.9 & 81.3 & 99.6 & 99.9 \\ 
Cauchy Softmax       & 74.5 & 99.6 & 99.5 & 99.2 \\ 
Normalized Dot Prod  & 62.6 & 92.6 & 99.3 & 99.1 \\ \bottomrule
\end{tabular}
\caption{Performance using various different similarity metrics}
\label{tab:similarity_measures}
\end{table}
}
\end{document}

%% file: math_commands.tex
\usepackage{amsmath,amsfonts,bm}

\def\eqref#1{equation~\ref{#1}}
\def\1{\bm{1}}

\def\vq{{\bm{q}}}

\def\vx{{\bm{x}}}

\def\mK{{\bm{K}}}

\def\mQ{{\bm{Q}}}

\def\mV{{\bm{V}}}

\DeclareMathAlphabet{\mathsfit}{\encodingdefault}{\sfdefault}{m}{sl}
\SetMathAlphabet{\mathsfit}{bold}{\encodingdefault}{\sfdefault}{bx}{n}

%% file: main.bbl
\begin{thebibliography}{55}
\providecommand{\natexlab}[1]{#1}
\providecommand{\url}[1]{\texttt{#1}}
\expandafter\ifx\csname urlstyle\endcsname\relax
  \providecommand{\doi}[1]{doi: #1}\else
  \providecommand{\doi}{doi: \begingroup \urlstyle{rm}\Url}\fi

\bibitem[Abdi \& Williams(2010)Abdi and Williams]{abdi2010principal}
Herv{\'e} Abdi and Lynne~J Williams.
\newblock Principal component analysis.
\newblock \emph{Wiley interdisciplinary reviews: computational statistics}, 2\penalty0 (4):\penalty0 433--459, 2010.

\bibitem[Arora et~al.(2024{\natexlab{a}})Arora, Eyuboglu, Timalsina, Johnson, Poli, Zou, Rudra, and Re]{arora2024zoology}
Simran Arora, Sabri Eyuboglu, Aman Timalsina, Isys Johnson, Michael Poli, James Zou, Atri Rudra, and Christopher Re.
\newblock Zoology: Measuring and improving recall in efficient language models.
\newblock In \emph{The Twelfth International Conference on Learning Representations}, 2024{\natexlab{a}}.
\newblock URL \url{https://openreview.net/forum?id=LY3ukUANko}.

\bibitem[Arora et~al.(2024{\natexlab{b}})Arora, Eyuboglu, Zhang, Timalsina, Alberti, Zou, Rudra, and Re]{arorasimple}
Simran Arora, Sabri Eyuboglu, Michael Zhang, Aman Timalsina, Silas Alberti, James Zou, Atri Rudra, and Christopher Re.
\newblock Simple linear attention language models balance the recall-throughput tradeoff.
\newblock In \emph{Forty-first International Conference on Machine Learning}, 2024{\natexlab{b}}.

\bibitem[Bahdanau et~al.(2015)Bahdanau, Cho, and Bengio]{bahdanau2016nmt}
Dzmitry Bahdanau, Kyunghyun Cho, and Yoshua Bengio.
\newblock Neural machine translation by jointly learning to align and translate.
\newblock In \emph{International Conference on Learning Representations}, 2015.

\bibitem[Beltagy et~al.(2020)Beltagy, Peters, and Cohan]{beltagy2020longformer}
Iz~Beltagy, Matthew~E Peters, and Arman Cohan.
\newblock Longformer: The long-document transformer.
\newblock \emph{arXiv preprint arXiv:2004.05150}, 2020.

\bibitem[Bertsch et~al.(2023)Bertsch, Alon, Neubig, and Gormley]{bertsch2023unlimiformer}
Amanda Bertsch, Uri Alon, Graham Neubig, and Matthew~R. Gormley.
\newblock Unlimiformer: Long-range transformers with unlimited length input.
\newblock In \emph{Thirty-seventh Conference on Neural Information Processing Systems}, 2023.
\newblock URL \url{https://openreview.net/forum?id=lJWUJWLCJo}.

\bibitem[Biderman et~al.(2023)Biderman, Schoelkopf, Anthony, Bradley, O’Brien, Hallahan, Khan, Purohit, Prashanth, Raff, et~al.]{biderman2023pythia}
Stella Biderman, Hailey Schoelkopf, Quentin~Gregory Anthony, Herbie Bradley, Kyle O’Brien, Eric Hallahan, Mohammad~Aflah Khan, Shivanshu Purohit, USVSN~Sai Prashanth, Edward Raff, et~al.
\newblock Pythia: A suite for analyzing large language models across training and scaling.
\newblock In \emph{International Conference on Machine Learning}, pp.\  2397--2430. PMLR, 2023.

\bibitem[Billingsley(1986)]{Bill86}
Patrick Billingsley.
\newblock \emph{Probability and Measure}.
\newblock John Wiley and Sons, second edition, 1986.

\bibitem[Brooks et~al.(2024)Brooks, Peebles, Holmes, DePue, Guo, Jing, Schnurr, Taylor, Luhman, Luhman, Ng, Wang, and Ramesh]{sora}
Tim Brooks, Bill Peebles, Connor Holmes, Will DePue, Yufei Guo, Li~Jing, David Schnurr, Joe Taylor, Troy Luhman, Eric Luhman, Clarence Ng, Ricky Wang, and Aditya Ramesh.
\newblock Video generation models as world simulators, 2024.
\newblock URL \url{https://openai.com/research/video-generation-models-as-world-simulators}.

\bibitem[Brown et~al.(2020)Brown, Mann, Ryder, Subbiah, Kaplan, Dhariwal, Neelakantan, Shyam, Sastry, Askell, Agarwal, Herbert-Voss, Krueger, Henighan, Child, Ramesh, Ziegler, Wu, Winter, Hesse, Chen, Sigler, Litwin, Gray, Chess, Clark, Berner, McCandlish, Radford, Sutskever, and Amodei]{brown2020language}
Tom~B. Brown, Benjamin Mann, Nick Ryder, Melanie Subbiah, Jared Kaplan, Prafulla Dhariwal, Arvind Neelakantan, Pranav Shyam, Girish Sastry, Amanda Askell, Sandhini Agarwal, Ariel Herbert-Voss, Gretchen Krueger, Tom Henighan, Rewon Child, Aditya Ramesh, Daniel~M. Ziegler, Jeffrey Wu, Clemens Winter, Christopher Hesse, Mark Chen, Eric Sigler, Mateusz Litwin, Scott Gray, Benjamin Chess, Jack Clark, Christopher Berner, Sam McCandlish, Alec Radford, Ilya Sutskever, and Dario Amodei.
\newblock Language models are few-shot learners.
\newblock In \emph{Proceedings of the 34th International Conference on Neural Information Processing Systems}, NIPS '20, Red Hook, NY, USA, 2020. Curran Associates Inc.
\newblock ISBN 9781713829546.

\bibitem[Chen et~al.(2021)Chen, Zeng, Ji, and Yang]{Skyformer}
Yifan Chen, Qi~Zeng, Heng Ji, and Yun Yang.
\newblock Skyformer: Remodel self-attention with gaussian kernel and nystr{\"o}m method.
\newblock In A.~Beygelzimer, Y.~Dauphin, P.~Liang, and J.~Wortman Vaughan (eds.), \emph{Advances in Neural Information Processing Systems}, 2021.
\newblock URL \url{https://openreview.net/forum?id=pZCYG7gjkKz}.

\bibitem[Child et~al.(2019)Child, Gray, Radford, and Sutskever]{child2019generating}
Rewon Child, Scott Gray, Alec Radford, and Ilya Sutskever.
\newblock Generating long sequences with sparse transformers.
\newblock \emph{arXiv preprint arXiv:1904.10509}, 2019.

\bibitem[Choromanski et~al.(2021)Choromanski, Likhosherstov, Dohan, Song, Gane, Sarlos, Hawkins, Davis, Mohiuddin, Kaiser, Belanger, Colwell, and Weller]{choromanski2020rethinking}
Krzysztof~Marcin Choromanski, Valerii Likhosherstov, David Dohan, Xingyou Song, Andreea Gane, Tamas Sarlos, Peter Hawkins, Jared~Quincy Davis, Afroz Mohiuddin, Lukasz Kaiser, David~Benjamin Belanger, Lucy~J Colwell, and Adrian Weller.
\newblock Rethinking attention with performers.
\newblock In \emph{International Conference on Learning Representations}, 2021.
\newblock URL \url{https://openreview.net/forum?id=Ua6zuk0WRH}.

\bibitem[Cover \& Thomas(2006)Cover and Thomas]{thomas2006elements}
Thomas~M. Cover and Joy~A. Thomas.
\newblock \emph{Elements of information theory}.
\newblock John Wiley \& Sons, 2006.

\bibitem[Dao(2024)]{dao2023flashattention2}
Tri Dao.
\newblock Flash{A}ttention-2: Faster attention with better parallelism and work partitioning.
\newblock In \emph{International Conference on Learning Representations (ICLR)}, 2024.

\bibitem[Dao et~al.(2022)Dao, Fu, Ermon, Rudra, and R{\'e}]{dao2022flashattention}
Tri Dao, Daniel~Y. Fu, Stefano Ermon, Atri Rudra, and Christopher R{\'e}.
\newblock Flash{A}ttention: Fast and memory-efficient exact attention with {IO}-awareness.
\newblock In \emph{Advances in Neural Information Processing Systems (NeurIPS)}, 2022.

\bibitem[Devlin et~al.(2019)Devlin, Chang, Lee, and Toutanova]{devlin2019bert}
Jacob Devlin, Ming-Wei Chang, Kenton Lee, and Kristina Toutanova.
\newblock {BERT}: Pre-training of deep bidirectional transformers for language understanding.
\newblock In Jill Burstein, Christy Doran, and Thamar Solorio (eds.), \emph{Proceedings of the 2019 Conference of the North {A}merican Chapter of the Association for Computational Linguistics: Human Language Technologies, Volume 1 (Long and Short Papers)}, pp.\  4171--4186, Minneapolis, Minnesota, June 2019. Association for Computational Linguistics.
\newblock \doi{10.18653/v1/N19-1423}.
\newblock URL \url{https://aclanthology.org/N19-1423}.

\bibitem[Dosovitskiy et~al.(2021)Dosovitskiy, Beyer, Kolesnikov, Weissenborn, Zhai, Unterthiner, Dehghani, Minderer, Heigold, Gelly, Uszkoreit, and Houlsby]{dosovitskiy2020image}
Alexey Dosovitskiy, Lucas Beyer, Alexander Kolesnikov, Dirk Weissenborn, Xiaohua Zhai, Thomas Unterthiner, Mostafa Dehghani, Matthias Minderer, Georg Heigold, Sylvain Gelly, Jakob Uszkoreit, and Neil Houlsby.
\newblock An image is worth 16x16 words: Transformers for image recognition at scale.
\newblock In \emph{International Conference on Learning Representations}, 2021.
\newblock URL \url{https://openreview.net/forum?id=YicbFdNTTy}.

\bibitem[Dugundji(1966)]{dugundji1989topology}
James Dugundji.
\newblock \emph{Topology}.
\newblock Allyn and Bacon, 1966.

\bibitem[Fang et~al.(2022)Fang, Wen, Kang, and Liu]{fang2022structure}
Ruiyi Fang, Liangjian Wen, Zhao Kang, and Jianzhuang Liu.
\newblock Structure-preserving graph representation learning.
\newblock In \emph{2022 IEEE International Conference on Data Mining (ICDM)}, pp.\  927--932. IEEE, 2022.

\bibitem[Fang et~al.(2025)Fang, Li, Kang, Zeng, Pu, Dashtbayaz, Wang, and Ling]{fang2025benefits}
Ruiyi Fang, Bingheng Li, Zhao Kang, Qiuhao Zeng, Ruizhi Pu, Nima~Hosseini Dashtbayaz, Boyu Wang, and Charles Ling.
\newblock On the benefits of attribute-driven graph domain adaptation.
\newblock \emph{The Thirteenth International Conference on Learning Representations}, 2025.

\bibitem[Gu \& Dao(2024)Gu and Dao]{gu2024mamba}
Albert Gu and Tri Dao.
\newblock Mamba: Linear-time sequence modeling with selective state spaces.
\newblock In \emph{First Conference on Language Modeling}, 2024.
\newblock URL \url{https://openreview.net/forum?id=tEYskw1VY2}.

\bibitem[Gupta et~al.(2021)Gupta, Dar, Goodman, Ciprut, and Berant]{gupta2021memory}
Ankit Gupta, Guy Dar, Shaya Goodman, David Ciprut, and Jonathan Berant.
\newblock Memory-efficient transformers via top-k attention.
\newblock In Nafise~Sadat Moosavi, Iryna Gurevych, Angela Fan, Thomas Wolf, Yufang Hou, Ana Marasovi{\'c}, and Sujith Ravi (eds.), \emph{Proceedings of the Second Workshop on Simple and Efficient Natural Language Processing}, pp.\  39--52, Virtual, November 2021. Association for Computational Linguistics.
\newblock \doi{10.18653/v1/2021.sustainlp-1.5}.
\newblock URL \url{https://aclanthology.org/2021.sustainlp-1.5}.

\bibitem[Ho et~al.(2020)Ho, Kalchbrenner, Weissenborn, and Salimans]{ho2020axial}
Jonathan Ho, Nal Kalchbrenner, Dirk Weissenborn, and Tim Salimans.
\newblock Axial attention in multidimensional transformers, 2020.
\newblock URL \url{https://openreview.net/forum?id=H1e5GJBtDr}.

\bibitem[Jiang et~al.(2024)Jiang, Sablayrolles, Roux, Mensch, Savary, Bamford, Chaplot, de~Las~Casas, Hanna, Bressand, Lengyel, Bour, Lample, Lavaud, Saulnier, Lachaux, Stock, Subramanian, Yang, Antoniak, Scao, Gervet, Lavril, Wang, Lacroix, and Sayed]{jiang_mixtral_2024}
Albert~Q. Jiang, Alexandre Sablayrolles, Antoine Roux, Arthur Mensch, Blanche Savary, Chris Bamford, Devendra~Singh Chaplot, Diego de~Las~Casas, Emma~Bou Hanna, Florian Bressand, Gianna Lengyel, Guillaume Bour, Guillaume Lample, L{\'{e}}lio~Renard Lavaud, Lucile Saulnier, Marie{-}Anne Lachaux, Pierre Stock, Sandeep Subramanian, Sophia Yang, Szymon Antoniak, Teven~Le Scao, Th{\'{e}}ophile Gervet, Thibaut Lavril, Thomas Wang, Timoth{\'{e}}e Lacroix, and William~El Sayed.
\newblock Mixtral of experts.
\newblock \emph{CoRR}, abs/2401.04088, 2024.
\newblock \doi{10.48550/ARXIV.2401.04088}.
\newblock URL \url{https://doi.org/10.48550/arXiv.2401.04088}.

\bibitem[Johnson et~al.(1986)Johnson, Lindenstrauss, and Schechtman]{johnson1986extensions}
William~B Johnson, Joram Lindenstrauss, and Gideon Schechtman.
\newblock Extensions of lipschitz maps into banach spaces.
\newblock \emph{Israel Journal of Mathematics}, 54\penalty0 (2):\penalty0 129--138, 1986.

\bibitem[Jurafsky \& Martin(2024)Jurafsky and Martin]{jm3}
Daniel Jurafsky and James~H. Martin.
\newblock \emph{Speech and Language Processing: An Introduction to Natural Language Processing, Computational Linguistics, and Speech Recognition with Language Models}.
\newblock Pearson, 3rd edition, 2024.
\newblock URL \url{https://web.stanford.edu/~jurafsky/slp3/}.

\bibitem[Kitaev et~al.(2020)Kitaev, Kaiser, and Levskaya]{kitaev2020reformer}
Nikita Kitaev, Lukasz Kaiser, and Anselm Levskaya.
\newblock Reformer: The efficient transformer.
\newblock In \emph{International Conference on Learning Representations}, 2020.
\newblock URL \url{https://openreview.net/forum?id=rkgNKkHtvB}.

\bibitem[Mao et~al.(2024)Mao, Ester, and Li]{mao2024iceformer}
Yuzhen Mao, Martin Ester, and Ke~Li.
\newblock Iceformer: Accelerated inference with long-sequence transformers on {CPU}s.
\newblock In \emph{The Twelfth International Conference on Learning Representations}, 2024.
\newblock URL \url{https://openreview.net/forum?id=6RR3wU4mSZ}.

\bibitem[McInnes et~al.(2018)McInnes, Healy, Saul, and Grossberger]{mcinnes2018umap}
Leland McInnes, John Healy, Nathaniel Saul, and Lukas Grossberger.
\newblock Umap: Uniform manifold approximation and projection.
\newblock \emph{The Journal of Open Source Software}, 3\penalty0 (29):\penalty0 861, 2018.

\bibitem[Merity et~al.(2017)Merity, Xiong, Bradbury, and Socher]{merity2016pointer}
Stephen Merity, Caiming Xiong, James Bradbury, and Richard Socher.
\newblock Pointer sentinel mixture models.
\newblock In \emph{International Conference on Learning Representations}, 2017.
\newblock URL \url{https://openreview.net/forum?id=Byj72udxe}.

\bibitem[OpenAI et~al.(2024)OpenAI, Achiam, Adler, Agarwal, Ahmad, Akkaya, Aleman, Almeida, Altenschmidt, Altman, Anadkat, Avila, Babuschkin, Balaji, Balcom, Baltescu, Bao, Bavarian, Belgum, Bello, Berdine, Bernadett-Shapiro, Berner, Bogdonoff, Boiko, Boyd, Brakman, Brockman, Brooks, Brundage, Button, Cai, Campbell, Cann, Carey, Carlson, Carmichael, Chan, Chang, Chantzis, Chen, Chen, Chen, Chen, Chen, Chess, Cho, Chu, Chung, Cummings, Currier, Dai, Decareaux, Degry, Deutsch, Deville, Dhar, Dohan, Dowling, Dunning, Ecoffet, Eleti, Eloundou, Farhi, Fedus, Felix, Fishman, Forte, Fulford, Gao, Georges, Gibson, Goel, Gogineni, Goh, Gontijo-Lopes, Gordon, Grafstein, Gray, Greene, Gross, Gu, Guo, Hallacy, Han, Harris, He, Heaton, Heidecke, Hesse, Hickey, Hickey, Hoeschele, Houghton, Hsu, Hu, Hu, Huizinga, Jain, Jain, Jang, Jiang, Jiang, Jin, Jin, Jomoto, Jonn, Jun, Kaftan, Łukasz Kaiser, Kamali, Kanitscheider, Keskar, Khan, Kilpatrick, Kim, Kim, Kim, Kirchner, Kiros, Knight, Kokotajlo, Łukasz Kondraciuk, Kondrich,
  Konstantinidis, Kosic, Krueger, Kuo, Lampe, Lan, Lee, Leike, Leung, Levy, Li, Lim, Lin, Lin, Litwin, Lopez, Lowe, Lue, Makanju, Malfacini, Manning, Markov, Markovski, Martin, Mayer, Mayne, McGrew, McKinney, McLeavey, McMillan, McNeil, Medina, Mehta, Menick, Metz, Mishchenko, Mishkin, Monaco, Morikawa, Mossing, Mu, Murati, Murk, Mély, Nair, Nakano, Nayak, Neelakantan, Ngo, Noh, Ouyang, O'Keefe, Pachocki, Paino, Palermo, Pantuliano, Parascandolo, Parish, Parparita, Passos, Pavlov, Peng, Perelman, de~Avila Belbute~Peres, Petrov, de~Oliveira~Pinto, Michael, Pokorny, Pokrass, Pong, Powell, Power, Power, Proehl, Puri, Radford, Rae, Ramesh, Raymond, Real, Rimbach, Ross, Rotsted, Roussez, Ryder, Saltarelli, Sanders, Santurkar, Sastry, Schmidt, Schnurr, Schulman, Selsam, Sheppard, Sherbakov, Shieh, Shoker, Shyam, Sidor, Sigler, Simens, Sitkin, Slama, Sohl, Sokolowsky, Song, Staudacher, Such, Summers, Sutskever, Tang, Tezak, Thompson, Tillet, Tootoonchian, Tseng, Tuggle, Turley, Tworek, Uribe, Vallone, Vijayvergiya,
  Voss, Wainwright, Wang, Wang, Wang, Ward, Wei, Weinmann, Welihinda, Welinder, Weng, Weng, Wiethoff, Willner, Winter, Wolrich, Wong, Workman, Wu, Wu, Wu, Xiao, Xu, Yoo, Yu, Yuan, Zaremba, Zellers, Zhang, Zhang, Zhao, Zheng, Zhuang, Zhuk, and Zoph]{openai_gpt-4_2023}
OpenAI, Josh Achiam, Steven Adler, Sandhini Agarwal, Lama Ahmad, Ilge Akkaya, Florencia~Leoni Aleman, Diogo Almeida, Janko Altenschmidt, Sam Altman, Shyamal Anadkat, Red Avila, Igor Babuschkin, Suchir Balaji, Valerie Balcom, Paul Baltescu, Haiming Bao, Mohammad Bavarian, Jeff Belgum, Irwan Bello, Jake Berdine, Gabriel Bernadett-Shapiro, Christopher Berner, Lenny Bogdonoff, Oleg Boiko, Madelaine Boyd, Anna-Luisa Brakman, Greg Brockman, Tim Brooks, Miles Brundage, Kevin Button, Trevor Cai, Rosie Campbell, Andrew Cann, Brittany Carey, Chelsea Carlson, Rory Carmichael, Brooke Chan, Che Chang, Fotis Chantzis, Derek Chen, Sully Chen, Ruby Chen, Jason Chen, Mark Chen, Ben Chess, Chester Cho, Casey Chu, Hyung~Won Chung, Dave Cummings, Jeremiah Currier, Yunxing Dai, Cory Decareaux, Thomas Degry, Noah Deutsch, Damien Deville, Arka Dhar, David Dohan, Steve Dowling, Sheila Dunning, Adrien Ecoffet, Atty Eleti, Tyna Eloundou, David Farhi, Liam Fedus, Niko Felix, Simón~Posada Fishman, Juston Forte, Isabella Fulford, Leo
  Gao, Elie Georges, Christian Gibson, Vik Goel, Tarun Gogineni, Gabriel Goh, Rapha Gontijo-Lopes, Jonathan Gordon, Morgan Grafstein, Scott Gray, Ryan Greene, Joshua Gross, Shixiang~Shane Gu, Yufei Guo, Chris Hallacy, Jesse Han, Jeff Harris, Yuchen He, Mike Heaton, Johannes Heidecke, Chris Hesse, Alan Hickey, Wade Hickey, Peter Hoeschele, Brandon Houghton, Kenny Hsu, Shengli Hu, Xin Hu, Joost Huizinga, Shantanu Jain, Shawn Jain, Joanne Jang, Angela Jiang, Roger Jiang, Haozhun Jin, Denny Jin, Shino Jomoto, Billie Jonn, Heewoo Jun, Tomer Kaftan, Łukasz Kaiser, Ali Kamali, Ingmar Kanitscheider, Nitish~Shirish Keskar, Tabarak Khan, Logan Kilpatrick, Jong~Wook Kim, Christina Kim, Yongjik Kim, Jan~Hendrik Kirchner, Jamie Kiros, Matt Knight, Daniel Kokotajlo, Łukasz Kondraciuk, Andrew Kondrich, Aris Konstantinidis, Kyle Kosic, Gretchen Krueger, Vishal Kuo, Michael Lampe, Ikai Lan, Teddy Lee, Jan Leike, Jade Leung, Daniel Levy, Chak~Ming Li, Rachel Lim, Molly Lin, Stephanie Lin, Mateusz Litwin, Theresa Lopez, Ryan
  Lowe, Patricia Lue, Anna Makanju, Kim Malfacini, Sam Manning, Todor Markov, Yaniv Markovski, Bianca Martin, Katie Mayer, Andrew Mayne, Bob McGrew, Scott~Mayer McKinney, Christine McLeavey, Paul McMillan, Jake McNeil, David Medina, Aalok Mehta, Jacob Menick, Luke Metz, Andrey Mishchenko, Pamela Mishkin, Vinnie Monaco, Evan Morikawa, Daniel Mossing, Tong Mu, Mira Murati, Oleg Murk, David Mély, Ashvin Nair, Reiichiro Nakano, Rajeev Nayak, Arvind Neelakantan, Richard Ngo, Hyeonwoo Noh, Long Ouyang, Cullen O'Keefe, Jakub Pachocki, Alex Paino, Joe Palermo, Ashley Pantuliano, Giambattista Parascandolo, Joel Parish, Emy Parparita, Alex Passos, Mikhail Pavlov, Andrew Peng, Adam Perelman, Filipe de~Avila Belbute~Peres, Michael Petrov, Henrique~Ponde de~Oliveira~Pinto, Michael, Pokorny, Michelle Pokrass, Vitchyr~H. Pong, Tolly Powell, Alethea Power, Boris Power, Elizabeth Proehl, Raul Puri, Alec Radford, Jack Rae, Aditya Ramesh, Cameron Raymond, Francis Real, Kendra Rimbach, Carl Ross, Bob Rotsted, Henri Roussez,
  Nick Ryder, Mario Saltarelli, Ted Sanders, Shibani Santurkar, Girish Sastry, Heather Schmidt, David Schnurr, John Schulman, Daniel Selsam, Kyla Sheppard, Toki Sherbakov, Jessica Shieh, Sarah Shoker, Pranav Shyam, Szymon Sidor, Eric Sigler, Maddie Simens, Jordan Sitkin, Katarina Slama, Ian Sohl, Benjamin Sokolowsky, Yang Song, Natalie Staudacher, Felipe~Petroski Such, Natalie Summers, Ilya Sutskever, Jie Tang, Nikolas Tezak, Madeleine~B. Thompson, Phil Tillet, Amin Tootoonchian, Elizabeth Tseng, Preston Tuggle, Nick Turley, Jerry Tworek, Juan Felipe~Cerón Uribe, Andrea Vallone, Arun Vijayvergiya, Chelsea Voss, Carroll Wainwright, Justin~Jay Wang, Alvin Wang, Ben Wang, Jonathan Ward, Jason Wei, CJ~Weinmann, Akila Welihinda, Peter Welinder, Jiayi Weng, Lilian Weng, Matt Wiethoff, Dave Willner, Clemens Winter, Samuel Wolrich, Hannah Wong, Lauren Workman, Sherwin Wu, Jeff Wu, Michael Wu, Kai Xiao, Tao Xu, Sarah Yoo, Kevin Yu, Qiming Yuan, Wojciech Zaremba, Rowan Zellers, Chong Zhang, Marvin Zhang, Shengjia
  Zhao, Tianhao Zheng, Juntang Zhuang, William Zhuk, and Barret Zoph.
\newblock Gpt-4 technical report, 2024.
\newblock URL \url{https://arxiv.org/abs/2303.08774}.

\bibitem[Ott et~al.(2018)Ott, Edunov, Grangier, and Auli]{ott2018scaling}
Myle Ott, Sergey Edunov, David Grangier, and Michael Auli.
\newblock Scaling neural machine translation.
\newblock In Ond{\v{r}}ej Bojar, Rajen Chatterjee, Christian Federmann, Mark Fishel, Yvette Graham, Barry Haddow, Matthias Huck, Antonio~Jimeno Yepes, Philipp Koehn, Christof Monz, Matteo Negri, Aur{\'e}lie N{\'e}v{\'e}ol, Mariana Neves, Matt Post, Lucia Specia, Marco Turchi, and Karin Verspoor (eds.), \emph{Proceedings of the Third Conference on Machine Translation: Research Papers}, pp.\  1--9, Brussels, Belgium, October 2018. Association for Computational Linguistics.
\newblock \doi{10.18653/v1/W18-6301}.
\newblock URL \url{https://aclanthology.org/W18-6301}.

\bibitem[Paszke et~al.(2019)Paszke, Gross, Massa, Lerer, Bradbury, Chanan, Killeen, Lin, Gimelshein, Antiga, Desmaison, K{\"{o}}pf, Yang, DeVito, Raison, Tejani, Chilamkurthy, Steiner, Fang, Bai, and Chintala]{pytorch}
Adam Paszke, Sam Gross, Francisco Massa, Adam Lerer, James Bradbury, Gregory Chanan, Trevor Killeen, Zeming Lin, Natalia Gimelshein, Luca Antiga, Alban Desmaison, Andreas K{\"{o}}pf, Edward~Z. Yang, Zachary DeVito, Martin Raison, Alykhan Tejani, Sasank Chilamkurthy, Benoit Steiner, Lu~Fang, Junjie Bai, and Soumith Chintala.
\newblock Pytorch: An imperative style, high-performance deep learning library.
\newblock In Hanna~M. Wallach, Hugo Larochelle, Alina Beygelzimer, Florence d'Alch{\'{e}}{-}Buc, Emily~B. Fox, and Roman Garnett (eds.), \emph{Advances in Neural Information Processing Systems 32: Annual Conference on Neural Information Processing Systems 2019, NeurIPS 2019, December 8-14, 2019, Vancouver, BC, Canada}, pp.\  8024--8035, 2019.
\newblock URL \url{https://proceedings.neurips.cc/paper/2019/hash/bdbca288fee7f92f2bfa9f7012727740-Abstract.html}.

\bibitem[Qin et~al.(2022{\natexlab{a}})Qin, Han, Sun, Li, Kong, Barnes, and Zhong]{devil_linear_attn}
Zhen Qin, Xiaodong Han, Weixuan Sun, Dongxu Li, Lingpeng Kong, Nick Barnes, and Yiran Zhong.
\newblock The devil in linear transformer.
\newblock In Yoav Goldberg, Zornitsa Kozareva, and Yue Zhang (eds.), \emph{Proceedings of the 2022 Conference on Empirical Methods in Natural Language Processing, {EMNLP} 2022, Abu Dhabi, United Arab Emirates, December 7-11, 2022}, pp.\  7025--7041. Association for Computational Linguistics, 2022{\natexlab{a}}.
\newblock \doi{10.18653/V1/2022.EMNLP-MAIN.473}.
\newblock URL \url{https://doi.org/10.18653/v1/2022.emnlp-main.473}.

\bibitem[Qin et~al.(2022{\natexlab{b}})Qin, Sun, Deng, Li, Wei, Lv, Yan, Kong, and Zhong]{qincosformer}
Zhen Qin, Weixuan Sun, Hui Deng, Dongxu Li, Yunshen Wei, Baohong Lv, Junjie Yan, Lingpeng Kong, and Yiran Zhong.
\newblock cosformer: Rethinking softmax in attention.
\newblock In \emph{International Conference on Learning Representations}, 2022{\natexlab{b}}.
\newblock URL \url{https://openreview.net/forum?id=Bl8CQrx2Up4}.

\bibitem[Radford et~al.(2019)Radford, Wu, Child, Luan, Amodei, Sutskever, et~al.]{radford2019language}
Alec Radford, Jeffrey Wu, Rewon Child, David Luan, Dario Amodei, Ilya Sutskever, et~al.
\newblock Language models are unsupervised multitask learners, 2019.

\bibitem[Ramesh et~al.(2021)Ramesh, Pavlov, Goh, Gray, Voss, Radford, Chen, and Sutskever]{DALL-E}
Aditya Ramesh, Mikhail Pavlov, Gabriel Goh, Scott Gray, Chelsea Voss, Alec Radford, Mark Chen, and Ilya Sutskever.
\newblock Zero-shot text-to-image generation.
\newblock In Marina Meila and Tong Zhang (eds.), \emph{Proceedings of the 38th International Conference on Machine Learning, {ICML} 2021, 18-24 July 2021, Virtual Event}, volume 139 of \emph{Proceedings of Machine Learning Research}, pp.\  8821--8831. {PMLR}, 2021.
\newblock URL \url{http://proceedings.mlr.press/v139/ramesh21a.html}.

\bibitem[Shalev-Shwartz \& Ben-David(2014)Shalev-Shwartz and Ben-David]{shalev2014understanding}
Shai Shalev-Shwartz and Shai Ben-David.
\newblock \emph{Understanding Machine Learning: From Theory to Algorithms}.
\newblock Cambridge University Press, 2014.

\bibitem[Shen et~al.(2023)Shen, Guo, Tan, Tang, Wang, and Bian]{relu_attention}
Kai Shen, Junliang Guo, Xu~Tan, Siliang Tang, Rui Wang, and Jiang Bian.
\newblock A study on relu and softmax in transformer.
\newblock \emph{CoRR}, abs/2302.06461, 2023.
\newblock \doi{10.48550/ARXIV.2302.06461}.
\newblock URL \url{https://doi.org/10.48550/arXiv.2302.06461}.

\bibitem[Tay et~al.(2020)Tay, Bahri, Yang, Metzler, and Juan]{tay2020sparse}
Yi~Tay, Dara Bahri, Liu Yang, Donald Metzler, and Da-Cheng Juan.
\newblock Sparse {S}inkhorn attention.
\newblock In Hal~Daumé III and Aarti Singh (eds.), \emph{Proceedings of the 37th International Conference on Machine Learning}, volume 119 of \emph{Proceedings of Machine Learning Research}, pp.\  9438--9447. PMLR, 13--18 Jul 2020.
\newblock URL \url{https://proceedings.mlr.press/v119/tay20a.html}.

\bibitem[Tay et~al.(2021)Tay, Dehghani, Abnar, Shen, Bahri, Pham, Rao, Yang, Ruder, and Metzler]{tay2020long}
Yi~Tay, Mostafa Dehghani, Samira Abnar, Yikang Shen, Dara Bahri, Philip Pham, Jinfeng Rao, Liu Yang, Sebastian Ruder, and Donald Metzler.
\newblock Long range arena : A benchmark for efficient transformers.
\newblock In \emph{International Conference on Learning Representations}, 2021.
\newblock URL \url{https://openreview.net/forum?id=qVyeW-grC2k}.

\bibitem[Tay et~al.(2022)Tay, Dehghani, Bahri, and Metzler]{survey}
Yi~Tay, Mostafa Dehghani, Dara Bahri, and Donald Metzler.
\newblock Efficient transformers: A survey.
\newblock \emph{ACM Comput. Surv.}, 55\penalty0 (6), December 2022.
\newblock ISSN 0360-0300.
\newblock \doi{10.1145/3530811}.
\newblock URL \url{https://doi.org/10.1145/3530811}.

\bibitem[Touvron et~al.(2021)Touvron, Cord, Douze, Massa, Sablayrolles, and Jegou]{touvron2021training}
Hugo Touvron, Matthieu Cord, Matthijs Douze, Francisco Massa, Alexandre Sablayrolles, and Herve Jegou.
\newblock Training data-efficient image transformers \& distillation through attention.
\newblock In Marina Meila and Tong Zhang (eds.), \emph{Proceedings of the 38th International Conference on Machine Learning}, volume 139 of \emph{Proceedings of Machine Learning Research}, pp.\  10347--10357. PMLR, 18--24 Jul 2021.
\newblock URL \url{https://proceedings.mlr.press/v139/touvron21a.html}.

\bibitem[Vaswani et~al.(2017)Vaswani, Shazeer, Parmar, Uszkoreit, Jones, Gomez, Kaiser, and Polosukhin]{vaswani2017attention}
Ashish Vaswani, Noam Shazeer, Niki Parmar, Jakob Uszkoreit, Llion Jones, Aidan~N. Gomez, \L{}ukasz Kaiser, and Illia Polosukhin.
\newblock Attention is all you need.
\newblock In I.~Guyon, U.~Von Luxburg, S.~Bengio, H.~Wallach, R.~Fergus, S.~Vishwanathan, and R.~Garnett (eds.), \emph{Advances in Neural Information Processing Systems}, volume~30. Curran Associates, Inc., 2017.
\newblock URL \url{https://proceedings.neurips.cc/paper_files/paper/2017/file/3f5ee243547dee91fbd053c1c4a845aa-Paper.pdf}.

\bibitem[Wang et~al.(2020)Wang, Li, Khabsa, Fang, and Ma]{wang2020linformer}
Sinong Wang, Belinda~Z Li, Madian Khabsa, Han Fang, and Hao Ma.
\newblock Linformer: Self-attention with linear complexity.
\newblock \emph{arXiv preprint arXiv:2006.04768}, 2020.

\bibitem[Xiong et~al.(2021)Xiong, Zeng, Chakraborty, Tan, Fung, Li, and Singh]{xiong2021nystromformer}
Yunyang Xiong, Zhanpeng Zeng, Rudrasis Chakraborty, Mingxing Tan, Glenn Fung, Yin Li, and Vikas Singh.
\newblock Nystr{\"o}mformer: A nystr{\"o}m-based algorithm for approximating self-attention.
\newblock In \emph{Proceedings of the AAAI Conference on Artificial Intelligence}, 2021.

\bibitem[Zaheer et~al.(2020)Zaheer, Guruganesh, Dubey, Ainslie, Alberti, Ontanon, Pham, Ravula, Wang, Yang, and Ahmed]{zaheer2020big}
Manzil Zaheer, Guru Guruganesh, Kumar~Avinava Dubey, Joshua Ainslie, Chris Alberti, Santiago Ontanon, Philip Pham, Anirudh Ravula, Qifan Wang, Li~Yang, and Amr Ahmed.
\newblock Big bird: Transformers for longer sequences.
\newblock 33:\penalty0 17283--17297, 2020.
\newblock URL \url{https://proceedings.neurips.cc/paper_files/paper/2020/file/c8512d142a2d849725f31a9a7a361ab9-Paper.pdf}.

\bibitem[Zeng et~al.(2023)Zeng, Wang, Zhou, Ling, and Wang]{zeng2023foresee}
Qiuhao Zeng, Wei Wang, Fan Zhou, Charles Ling, and Boyu Wang.
\newblock Foresee what you will learn: data augmentation for domain generalization in non-stationary environment.
\newblock In \emph{Proceedings of the AAAI conference on artificial intelligence}, volume~37, pp.\  11147--11155, 2023.

\bibitem[Zeng et~al.(2024{\natexlab{a}})Zeng, Huang, Chen, Ling, and Wang]{zeng2024towards}
QIUHAO Zeng, Long-Kai Huang, Qi~Chen, Charles~X Ling, and Boyu Wang.
\newblock Towards understanding evolving patterns in sequential data.
\newblock \emph{Advances in Neural Information Processing Systems}, 37:\penalty0 132747--132773, 2024{\natexlab{a}}.

\bibitem[Zeng et~al.(2024{\natexlab{b}})Zeng, Shui, Huang, Liu, Chen, Ling, and Wang]{zeng2024latent}
QIUHAO Zeng, Changjian Shui, Long-Kai Huang, Peng Liu, Xi~Chen, Charles Ling, and Boyu Wang.
\newblock Latent trajectory learning for limited timestamps under distribution shift over time.
\newblock In \emph{The Twelfth International Conference on Learning Representations}, 2024{\natexlab{b}}.
\newblock URL \url{https://openreview.net/forum?id=bTMMNT7IdW}.

\bibitem[Zhai et~al.(2023)Zhai, Likhomanenko, Littwin, Busbridge, Ramapuram, Zhang, Gu, and Susskind]{zhai2023stabilizing}
Shuangfei Zhai, Tatiana Likhomanenko, Etai Littwin, Dan Busbridge, Jason Ramapuram, Yizhe Zhang, Jiatao Gu, and Joshua~M. Susskind.
\newblock Stabilizing transformer training by preventing attention entropy collapse.
\newblock In Andreas Krause, Emma Brunskill, Kyunghyun Cho, Barbara Engelhardt, Sivan Sabato, and Jonathan Scarlett (eds.), \emph{Proceedings of the 40th International Conference on Machine Learning}, volume 202 of \emph{Proceedings of Machine Learning Research}, pp.\  40770--40803. PMLR, 23--29 Jul 2023.
\newblock URL \url{https://proceedings.mlr.press/v202/zhai23a.html}.

\bibitem[Zhang et~al.(2024)Zhang, Bhatia, Kumbong, and Re]{zhang2024the}
Michael Zhang, Kush Bhatia, Hermann Kumbong, and Christopher Re.
\newblock The hedgehog \& the porcupine: Expressive linear attentions with softmax mimicry.
\newblock In \emph{The Twelfth International Conference on Learning Representations}, 2024.
\newblock URL \url{https://openreview.net/forum?id=4g02l2N2Nx}.

\bibitem[Zhu \& Soricut(2021)Zhu and Soricut]{zhu-soricut-2021-h}
Zhenhai Zhu and Radu Soricut.
\newblock {H}-transformer-1{D}: Fast one-dimensional hierarchical attention for sequences.
\newblock In Chengqing Zong, Fei Xia, Wenjie Li, and Roberto Navigli (eds.), \emph{Proceedings of the 59th Annual Meeting of the Association for Computational Linguistics and the 11th International Joint Conference on Natural Language Processing (Volume 1: Long Papers)}, pp.\  3801--3815, Online, August 2021. Association for Computational Linguistics.
\newblock \doi{10.18653/v1/2021.acl-long.294}.
\newblock URL \url{https://aclanthology.org/2021.acl-long.294}.

\bibitem[Zhuoran et~al.(2021)Zhuoran, Mingyuan, Haiyu, Shuai, and Hongsheng]{roy2021efficient}
Shen Zhuoran, Zhang Mingyuan, Zhao Haiyu, Yi~Shuai, and Li~Hongsheng.
\newblock Efficient attention: Attention with linear complexities.
\newblock In \emph{2021 IEEE Winter Conference on Applications of Computer Vision (WACV)}, 2021.

\end{thebibliography}
